\documentclass[11pt,twoside]{article}
\usepackage[margin=1in]{geometry}
\usepackage{microtype}
\usepackage{graphicx}
\usepackage{booktabs} 
\usepackage{amsmath,amssymb,amsthm}
\usepackage{verbatim,makeidx}
\usepackage{algorithm}
\usepackage{algorithmic}
\usepackage{color}
\usepackage{colortbl}
\usepackage{dsfont}
\usepackage{caption}
\usepackage{subcaption}
\usepackage{graphicx}
\usepackage{verbatim}
\usepackage{pbox}
\usepackage{authblk}
\usepackage{bm}
\usepackage{enumitem}
\usepackage{comment}

\usepackage{hyperref}

\newcounter{ALC@tempcntr}

\theoremstyle{plain}
\newtheorem{theorem}{Theorem}
\newtheorem{proposition}{Proposition}
\newtheorem{lemma}{Lemma}

\newtheorem{remark}{Remark}

\newtheorem{scheme}{Scheme}
\newtheorem{assumption}{Assumption}

\newcommand{\beq}{\begin{eqnarray}}
\newcommand{\eeq}{\end{eqnarray}}

\newcommand{\field}[1]{\mathbb{#1}}

\newcommand{\R}{\field{R}}

\newfont{\bbb}{msbm10 scaled 500}

\newfont{\bb}{msbm10 scaled 1100}

\newcommand{\bv}{{\bf b}}
\newcommand{\cv}{{\bf c}}

\newcommand{\gv}{{\bf g}}

\newcommand{\mv}{{\bf m}}

\newcommand{\wv}{{\bf w}}

\newcommand{\xv}{{\bf x}}
\newcommand{\yv}{{\bf y}}
\newcommand{\zv}{{\bf z}}

\newcommand{\Ac}{{\cal A}}
\newcommand{\Bc}{{\cal B}}
\newcommand{\Cc}{{\cal C}}

\newcommand{\Hc}{{\cal H}}

\newcommand{\Lc}{{\cal L}}

\newcommand{\Pc}{{\cal P}}

\newcommand{\Rc}{{\cal R}}
\newcommand{\Sc}{{\cal S}}
\newcommand{\Tc}{{\cal T}}
\newcommand{\Uc}{{\cal U}}

\newcommand{\remove}[1]{}

\newcommand{\avg}{{\mathbb E}}

\newcommand\reals{{\mathbb R}}

\theoremstyle{definition}

\theoremstyle{remark}

\newcommand{\latexe}{{\LaTeX\kern.125em2%
                      \lower.5ex\hbox{$\varepsilon$}}}

\chardef\bslash=`\\	

\makeatletter		
\def\square{\RIfM@\bgroup\else$\bgroup\aftergroup$\fi
\vcenter{\hrule\hbox{\vrule\@height.6em\kern.6em\vrule}
\hrule}\egroup}\makeatother\makeindex

\definecolor{OXO-emph}{RGB}{153,0,0}

\DeclareMathAlphabet{\mathpzc}{OT1}{pzc}{m}{it}

\setlist[itemize]{leftmargin=0.1in}
\setlist[enumerate]{leftmargin=0.1in}

\title{Robust Gradient Descent via Moment Encoding and LDPC Codes}

\title{Robust Gradient Descent via Moment Encoding with LDPC Codes}

\author[1]{Raj Kumar Maity}
\author[1,2]{Ankit Singh Rawat}
\author[1]{Arya Mazumdar}
\affil[1]{\normalsize College of Information and Computer Sciences, University of Massachusetts Amherst, Amherst, MA 01003, USA.}
\affil[2]{\normalsize Research Laboratory of Electronics, MIT, Cambridge, MA 02139, USA.\newline E-mail: rajkmaity@cs.umass.edu, asrawat@mit.edu, arya@cs.umass.edu.}

\begin{document}

\maketitle


\begin{abstract}
This paper considers the problem of implementing large-scale gradient descent algorithms in a distributed computing setting in the presence of {\em straggling} processors. To mitigate the effect of the stragglers, it has been previously proposed to encode the data with an erasure-correcting code and decode at the master server at the end of the computation. We, instead, propose to encode the second-moment of the data with a low density parity-check (LDPC) code. The iterative decoding algorithms for LDPC codes have very low computational overhead and the number of decoding iterations can be made to automatically adjust with the number of stragglers in the system. We show that for a random model for stragglers, the proposed moment encoding based gradient descent method can be viewed as the stochastic gradient descent method. This allows us to obtain convergence guarantees for the proposed solution. Furthermore, the proposed moment encoding based method is shown to outperform the existing schemes in a real distributed computing setup.
\end{abstract}

\section{Introduction}
\label{sec:intro}

The vast volumes of data at our disposal today has made many useful data-driven tasks possible, which were otherwise thought to be infeasible. The large scale of the data necessitates working with distributed computing setups as the computational power available at a single location is not sufficient to meet the strict performance requirements in many real-life systems. Moreover, in many settings, local compute and storage resources cannot simply accommodate the entire data being processed. 

As a general principle, the large-scale distributed computing setups (e.g.,~\cite{MapReduce, Spark}) divide the original problem at hand into many small tasks, which are assigned to many servers, namely workers. The master server then collects the outcomes of local computation at the workers (potentially over multiple rounds) and computes the final result. In practical systems, the process of collecting the outcomes from the workers is prone to unpredictable delays~\cite{DeanL13}. Such delays arise due to various reasons, including the slow-down at the workers and the congestion in the communication networks in the system. The workers that cannot provide the outcome of their local computation within a reasonable deadline due to these delays are termed {\em stragglers}. The presence of the stragglers can significantly degrade the performance of the system. Therefore, it is imperative that we address the variability in the response times of different components of the setup during the design of the computations tasks.

Multiple recent works explore the problem of mitigating the effect of stragglers. The replication schemes assign each task to multiple servers~\cite{ShahLR16,AnGSS13,WJW15}. This ensures that the task gets completed without significant delay if at least one of the servers processing the task is non-straggler. In \cite{LLPPR15}, Lee et al. explore the coding theoretic ideas that go beyond the replication schemes to address the issue of stragglers. In particular, they focus on linear computation, namely a matrix-vector product, and proposes to encode the columns of the matrix by a maximum distance separable (MDS) code to obtain a taller encoded matrix. The rows of the encoded matrix are distributed among the workers, who are responsible for computing the inner product of the rows assigned to them with the vector in question. The redundancy among the rows of the encoded matrix allows for computation of the intended matrix-vector product even if some of the servers fail to respond with the computation assigned to them. In \cite{DuttaCG16}, Dutta et al. further explore the problem of reliably computing a matrix-vector product while additionally requiring that the rows of the encoded matrix are sparse. This aims at reducing the computation at the workers and communication between the master and the workers which scale with the row-sparsity of the encoded matrix. The similar ideas for other computational tasks (e.g., matrix-matrix product and convolution between vectors) have been explored in \cite{LeeSR17, YuMA17, DuttaCG17}. Bitar et al. propose a scheme to securely compute matrix-vector product without revealing any information about the matrix to the workers~\cite{BitarPR17}. Another line of work that aims at minimizing communication during data shuffling by using coding techniques is presented in \cite{LiMA15, LiMYA17, LiMA16, LLPPR15} and reference therein.

\paragraph{Our contributions.}
We focus on the problem of fitting a structured linear model to the given data. In particular, given the features or data points $\{\xv_i\}_{i \in [m] :=\{1,\ldots, m\}} \subset \reals^{k}$ and the associated labels $\{y_i\}_{i \in [m]}\subset \reals$, we want to learn the model parameter $\bm{\theta}^{\ast}$ belonging to a structured set $\Theta \subset \reals^{k}$ so that
$
y_i = \xv_i^T\bm{\theta}^{\ast} + \epsilon_i,
$
for small modeling errors $\{\epsilon_i\}_{i \in [m]}$. In many applications, the prior knowledge about the structure of the model parameter (such as, sparsity and group sparsity) can be expressed with the help of a {\em regularizer} $\Rc:\reals^k \to \reals$ so that $\Theta \equiv \{\bm{\theta} \in \reals^k :  \Rc(\bm{\theta}) \leq R\}$, for $R \in \reals$. In such settings, the task of recovering $\bm{\theta}^{\ast}$ can be realized by solving the following optimization problem. 
\begin{align}
\label{eq:gLasso}
\min_{\bm{\theta}} \frac12 \sum_{i=1}^m (y_i -  \xv_i^T\bm{\theta})^2  \quad \text{ subject to } \quad \bm{\theta} \in \Theta = \{\bm{\theta}' \in \reals^k :  \Rc(\bm{\theta}') \leq R\}.
\end{align}
Note that the square loss -- employed in the optimization problem above-- is one of the most pervasive loss functions in machine learning, optimization, and statistics. A large class of estimation problems arising in practice, such as compressed sensing~\cite{figueiredo2007gradient}, dictionary learning \cite{mairal2009online}, and matrix completion \cite{koltchinskii2011nuclear}, can be solved as special cases of the general optimization problem outlined in \eqref{eq:gLasso}~\cite{tibshirani2015statistical}. 

Even though, we focus on the constrained optimization problem, our proposed solution easily extends to the unconstrained optimization problem $
\min_{\bm{\theta} \in \reals^k} \sum_{i=1}^m (y_i -  \xv_i^T\bm{\theta})^2/2 + \lambda\cdot\Rc(\bm{\theta}),  
$ 
where we incorporate the regularizer in the objective function with the help of a regularization parameter $\lambda \in \reals$.

We note that our proposed solution can also be employed to recover the structured model parameters for single-index or generalized linear models~\cite{kakade2011efficient}, where the given data fits a model of the form:
$
y_i = g(\xv_i^T\bm{\theta}^{\ast}) + \epsilon_i,
$
with $g:\reals\to \reals$ denoting a possibly unknown nonlinear link function. In this setting, the model parameter can again be recovered by solving a generalized LASSO in \eqref{eq:gLasso}~\cite{plan2016generalized}. 
We employ the {\em projected gradient descent} (PGD) method to solve the underlying optimization problem. In a distributed computing setup, the iterative optimization procedure is implemented as follows. The master server maintains an estimate for the model parameter. In each step, the master sends the current estimate to the workers. The workers then compute a partial value of the gradient based on the received estimate and send the outcome of their computations to the master. By combining the messages received from the {\em non-straggling} workers, the master computes the gradient and updates its current estimate for the model parameter. In this paper, our first contribution is to propose a preprocessing step which encodes the {\em second moment} of the data and then distributes the encoded moments among the workers. This way, there is some redundancy among the outcomes of the computation at the workers, which allows the master to obtain a good enough estimate of the gradient even when it does not receive the outcome of the computation assigned to the stragglers. 

We employ the low-density parity check (LDPC) codes to encode the second moment of the data points. As a result, the task of calculating the gradient at the master reduces to the task of decoding an LDPC codeword in the presence of {\em erasures}, where the erased locations depend on the identities of the stragglers. The reason for working with LDPC codes is that the iterative decoding algorithms for these codes provide us three-fold benefits. The decoder has very low computational complexity and can automatically adjust to the number of the stragglers with a small number of decoding iterations required if there are not too many stragglers present. Additionally, we can use the number of decoding iterations as a tuning parameter. Depending on the number of stragglers, we can run only those many decoding iterations that are sufficient to ensure the desired quality of the estimate of the gradient. In our setup, the number of erased coordinates of the gradient vector serves as a measure of its quality. Note that this measure is a {\em non-increasing} function of the number of decoding iterations. Finally, the MDS code based solutions provided in prior literature (such as, \cite{LLPPR15,YuMA17}) suffer from the issue of noise-stability resulting from the low condition number of Vandermonde matrices, which we bypass by considering LDPC matrices.

Furthermore, we show that for a random model for stragglers, the PGD method with the proposed moment encoding scheme can be viewed as the projected stochastic gradient descent (PSGD) method. We then use the convergence analysis for the PSGD method to establish the convergence guarantee for our proposed solution. This analysis clearly characterizes the advantage over non-redundant or replication based gradient descent method in terms of the decoding iterations employed in each step of the method. We also conduct a detailed performance evaluation of our solution on a real-life distributed computing framework (swarm2) at the University of Massachusetts Amherst  \cite{swarm2}. The performance results show that, as compared to the existing schemes, our proposed solution requires a smaller number of gradient steps in order to converge to the correct model parameter.
%

\paragraph{Comparison with other relevant works.}
In \cite{LLPPR15}, Lee et al.  focus on performing iterative gradient descent method in a distributed manner via repeatedly invoking their solution for coded computation of matrix-vector product. In this paper, we also rely on the coded computation of matrix-vector product to realize iterative gradient descent in a straggler tolerant manner. However, we encode the second moment matrix as opposed to the plain data matrix as done in \cite{LLPPR15}. This leads to reduced communication rounds. Furthermore, this also makes the analysis of the optimization procedure completely different from that in \cite{LLPPR15}. As another novel contribution, we utilize LDPC codes which, as discussed above, allow for both efficient decoding and control over the quality of the (approximate) gradient computed in each step of the optimization procedure. In \cite{KarakusSDY17}, Karakus et al. also study the problem of recovering the model parameter of a linear model by solving an alternative optimization problem where they encode both data points and their labels by the matrices with maximal (pairwise) incoherent columns. Again, our approach differs from theirs as we solve the original optimization problem itself and rely on moment encoding as opposed to data encoding. 

In \cite{TLDK17}, Tandon et al. propose a novel framework, namely {\em gradient coding}, to counter the effect of stragglers on the performance of the gradient descent method. The gradient coding framework is designed for general loss functions which decompose over the data points. The gradient coding essentially relies on replication by cleverly assigning the data points to multiple workers to evaluate partial gradients. The specific designs for the replication among servers in the gradient coding framework along with their performance analysis are presented in \cite{RavivTTD17, CharlesPE17, HalbawiRSH17}. {\color{black}Here, we note that the square loss that we consider does have the additive structure. However, employing the (replication based) gradient coding framework to square loss leads to inefficient utilization of the compute and the communication resources.} In \cite{YangGK17}, Yang et al. also study the iterative methods to solve linear inverse problems in the presence of stragglers. However, their setup significantly differs from our setup. In \cite{YangGK17}, multiple instances of the gradient descent method are run on different machines in a redundant manner such that each machine is responsible for locally solving an entire instance. Whereas in our setup, a single instance of the linear inverse problem is solved by a network of servers and each server communicates its partial results in each step of the gradient descent method. There exists a large literature dealing with various issues other than the stragglers in the context of distributed optimization and learning. We refer the readers to \cite{LLPPR15} for an excellent exposition of the literature. 

\paragraph{Organization.} We present the exact problem formulation along with the necessary background in Section~\ref{sec:sys}. We present the main contribution of this paper in Section~\ref{sec:EOP} where we describe the moment encoding based optimization scheme along with its convergence analysis. In Section~\ref{sec:simulations}, we perform an extensive evaluation of the proposed scheme in a real-life distributed computing setup and compare it with the prior work. We present a list of notations in Table~\ref{tab:comparison} for ease of reading.

\bgroup
\def\arraystretch{1.55}
\begin{table*}[ht!]
\caption{List of notation}
\label{tab:comparison}
\footnotesize
\centering
\begin{tabular}{|c|l|}
\hline \hline
$m$ & Number of samples/data points  \\
\hline
$k$  & Dimension of samples \\
\hline 
$(N,K)$ & Length and dimension of the employed code \\
\hline 
$n = N\frac{k}{K}$  & Length of the encoded vectors \\
\hline
$w$ & Number of worker servers \\
\hline 
$\alpha  = \frac{n}{w}$ & Number  of rows mapped to each server\\
\hline
$\bm{\theta}$ & Model parameter to be learnt \\
\hline
$\ell\big((y, \xv\big), \bm{\theta})$ & Loss associated with $\bm{\theta}$ for data point $\xv$ and label $y$\\
\hline
$\Lc(\bm{\theta})$ & Total empirical loss associated with $\bm{\theta}$ \\
\hline
$d$ & Index for LDPC decoding itrations \\
\hline 
 $D$ & Number of iteration of LDPC decoding during each gradient descent step \\
 \hline 
 $t$ & Index for gradient descent steps \\
 \hline 
 $T$ & Number of gradient descent steps \\
 \hline
 $\eta_t$ & Learning rate for gradient descent \\
 \hline 
 $\Sc_{t}$ & Set of servers available during the $t$-th gradient descent step \\
 \hline 
 $s$ & Number of stragglers \\
 \hline 
\end{tabular}
\end{table*}
\egroup


\section{System model and background}
\label{sec:sys}

Our distributed computing setup has $w$ {worker servers} and one {master server}. Performing large-scale computation in this setup involves dividing the desired computation problem into multiple small computation tasks that are assigned to the workers. The master then collects the outcomes of the tasks mapped to the workers and produces the final result. The overall computation may require multiple rounds of communication among the master and the workers. 

We are given $m$ data samples or feature vectors $\{\xv_i\}_{i \in [m]} \subset \R^k$ and their labels $\{y_i\}_{i \in [m]} \subset \R$. In this paper, we are mainly concerned with learning a structured linear model. In particular, we are interested in learning a vector $\bm{\theta}^{\ast} \in \Theta \equiv \{\bm{\theta}' \in \reals^k : \Rc(\bm{\theta}') \leq R\}$, for some regularizer $\Rc: \reals^k \to \reals$, such that the following total empirical loss is minimized.
\begin{align}
\label{eq:loss}
\Lc(\bm{\theta}) = \frac{1}{2}\|\mathbf{y} - X\bm{\theta}\|^2_2 
= \frac{1}{2}\sum_{i = 1}^{m} \big(y_i - \xv^T_i \bm{\theta}\big)^2,
\end{align} 
where $\yv = (y_1, y_2,\ldots, y_m)^T \in \R^{m}$ and $X = (\xv_1~\xv_2~\cdots~\xv_m)^T \in \R^{m \times k}$. Note that the gradient of the total empirical loss with respect to $\bm{\theta}$ has the following form.
\begin{align}
\label{eq:gradient}
\nabla_{\bm{\theta}}\Lc(\bm{\theta}) = \big(X^TX\bm{\theta} - X^T\yv\big).
\end{align}
In this paper, we rely on the PGD method to solve the underlying constrained optimization problem, which iteratively updates an estimate of $\bm{\theta}^{\ast}$.
Specifically, at the $t$-th step, the estimate $\bm{\theta}_t$ has the form
\begin{align}
\label{eq:PGD_iter}
\bm{\theta}_{t} = P_{\Theta}\big(\bm{\theta}_{t-1} - \eta_{t} \nabla_{\bm{\theta}}\Lc(\bm{\theta}_{t-1})\big),
\end{align}
where $\eta_t$ is the learning rate at the $t$-th step, which may potentially be independent of $t$; and the projection operator $P_{\Theta}:\reals^k \to \reals^k$ is defined to be
$\bm{\theta} \mapsto \arg \min_{\tilde{\bm{\theta}} \in \Theta}\ \|\bm{\theta} -\tilde{\bm{\theta}}\|_2^2.$

{\color{black}
\begin{remark}
In our proposed scheme, the master performs the projection step in \eqref{eq:PGD_iter}. Thus, we are mainly interested in the regularizers with computationally efficient projection operations. This is particularly true for decomposable regularizers, such as sparsity constraints.
\end{remark}
}

\paragraph{Preliminary: Linear codes.}
In this paper, we rely on linear codes to perform the overall computation on a distributed computing setup in redundant manner. The redundancy allows the master to realize the original computation task in a straggler tolerant manner. An $(n, k)$ linear code is simply a subspace of dimension $k$ belonging to an $n$-length vector space. In this paper, we focus on the vector space defined over the real numbers $\R$. Therefore, an $(n,k)$ linear code $\Cc$ forms a $k$-dimensional subspace in $\R^n$. Given an $k$-length {\em message vector} $\xv \in \R^{k}$, it can be {\em encoded} (or mapped) to a {\em codeword} from the code $\Cc$ with the help of a generator matrix $G \in \R^{n \times k}$ as 
$\cv = G\xv  \in \Cc.$
Thus, a linear code can be defined as 
$\Cc := \{ \cv \in \R^{n} : \cv = G\xv~\text{for some}~\xv \in \R^k\}.$
 Alternatively, a linear code can also be defined by a parity check matrix $H \in \R^{(n-k) \times n}$ as follows
$ \Cc := \{ \cv \in \R^{n} : H\cv = 0\}.$
 A generator matrix leads to a {\em systematic} encoding, if for each $\xv \in \R^{k}$, the message vector $\xv$ exactly appears as $k$ coordinates of the associated codeword $\cv = G \xv$. The redundancy introduced by mapping a $k$ dimensional vector $\xv$ to an $n$-dimensional vector $\cv$ with  $n > k$ allows one to recover $\xv$ from $\cv$ even when some of the coordinates of $\cv$ are missing. In particular, if the code $\Cc$ has minimum distance $d_{\min}$, then $\xv$ can be recovered 
 even if any $d_{\min}-1$ coordinates of $\cv$ are not available.

 \subsection{The data coding method of~\cite{LLPPR15} and the gradient coding approach of~\cite{TLDK17}}

An approach to run gradient descent in a distributed system using reliable distributed matrix multiplication as a building block was recently presented by Lee et al. \cite{LLPPR15}.
Note that, in the linear regression problem, computing the gradient of the total empirical loss involves computation of two matrix-vector products in each iteration (see \eqref{eq:gradient}), namely:
$
X\bm{\theta}_{t-1}$   and $X^T(X\bm{\theta}_{t-1}-\yv).
$
In  \cite{LLPPR15}, an   MDS-coded distributed algorithm for matrix multiplication was proposed. In this algorithm, to perform the matrix-vector product $A\xv$, the matrix $A$ is premultiplied by the generator matrix $G$ of an MDS code of proper dimensions to get $\tilde{A} = GA$. Each worker node then performs a single
 inner product (or a set of inner products) involving a row of $\tilde{A}$ and $\xv$. The results of these local computations are then sent to the master node.
 As long as the number of workers that successfully deliver their local computations within the deadline is more than a specified threshold (in other words, as long as the number of stragglers is within the erasure correcting capability of the MDS code given by $G$), the product $A\xv$ can be found at the master node. In each iteration of the gradient descent, the above matrix-vector product protocol is applied twice (see \cite{LLPPR15} for details) to compute $X\bm{\theta}_{t-1}$  and $X^T(X\bm{\theta}_{t-1}-\yv).$ This facilitates computation of the gradient in each iteration in the presence of the stragglers.

In \cite{TLDK17}, Tandon et al. propose a novel framework to exactly compute gradient of the underlying loss function in a distributed computation setup.
In particular, they consider a generic loss function which takes the following additive form. 
$\Lc(\bm{\theta}) = \sum_{i = 1}^{m}\ell\big((y_i, \xv_i), \bm{\theta}\big). $
For such a loss function, its gradient can be obtained as
\begin{align}
\label{eq:gen_gradient}
\nabla_{\bm{\theta}}\Lc(\bm{\theta}) = \sum_{i = 1}^{m} \nabla_{\bm{\theta}}\ell\big((y_i, \xv_i), \bm{\theta}\big).
\end{align}
In order to compute the gradient in a distributed manner, the samples and the corresponding labels are distributed among $w$ workers in a redundant manner. For $i \in [w]$, the samples and labels allocated to the $i$-th worker server are indexed by the set $\Ac_i \subseteq [m]$. 
Given the samples and labels indexed by the set $\Ac_i$, the $i$-th worker can compute the following components of the gradient (cf.~\eqref{eq:gen_gradient}).
\begin{align}
\label{eq:gradient_blocks}
\Bc_i := \big\{\nabla_{\bm{\theta}}\ell\big((y_j, \xv_j), \bm{\theta}\big)\big\}_{j \in \Ac_i} \subset \R^k.
\end{align}
Now, the $i$-th worker transmits a linear combination of the blocks in $\Bc_i$ to the master. In particular, the transmitted block can be represented as follows. 
\begin{align}
\label{eq:tx_gradient}
\zv_i = \sum_{j \in \Ac_i}b_{i, j}\nabla_{\bm{\theta}}\ell\big((y_j, \xv_j), \bm{\theta}\big) \in \R^k.
\end{align}
Equivalently, the transmitted blocks from all $w$ workers can be represented as the following $w \times k$ matrix.
\begin{align}
&Z = (\zv_1, \ldots, \zv_w)^T \nonumber \\
&= B \big(\nabla_{\bm{\theta}}\ell((y_1, \xv_1), \bm{\theta})~\cdots~\nabla_{\bm{\theta}}\ell((y_m, \xv_m), \bm{\theta})\big)^T,
\end{align}
where $B$ is an $w \times m$ matrix containing the coefficients associated with the transmission from $w$ workers (cf.~\eqref{eq:tx_gradient}). Note that, for $i \in [w]$, the support of the $i$-th row of the matrix $B$ is contained in the set $\Ac_i$.

Let $\Sc \subset [w]$ denote the set of indices of the workers that successfully deliver their local computations within the deadline. Assuming that we have $s$ straggling workers which do not respond with their intended transmission before the deadline, we have $|\Sc| = w - s$. Note that the master has following information at its disposal. 
\begin{align}
\label{eq:rx_gradient}
Z_{\Sc} = B_{\Sc} \big(\nabla_{\bm{\theta}}\ell((y_1, \xv_1), \bm{\theta})~\cdots~\nabla_{\bm{\theta}}\ell((y_m, \xv_m), \bm{\theta})\big)^T, 
\end{align}
where $Z_{\Sc}$ and $B_{\Sc}$ denote the sub-matrices formed by the rows indexed by $\Sc$ in $Z$ and $B$, respectively. In order to be able to obtain the gradient
\begin{align*}
\nabla_{\bm{\theta}}\Lc(\bm{\theta}) &= \sum_{i = 1}^{m} \nabla_{\bm{\theta}}\ell((y_i, \xv_i), \bm{\theta}) \nonumber \\
&= (1,\ldots, 1)\cdot \big(\nabla_{\bm{\theta}}\ell((y_1, \xv_1), \bm{\theta})~\cdots~\nabla_{\bm{\theta}}\ell((y_m, \xv_m), \bm{\theta})\big)^T,
\end{align*}
we require that the all ones vector $(1,\ldots, 1)$ belongs to the subspace spanned by the rows of the matrix $B_{\Sc}$. {\em Therefore, the design criterion in the gradient coding approach~\cite{TLDK17} is to find an allocation of the samples $\{\Ac_i\}_{i \in [s]}$ and the associated transmission matrix $B$ such that for every $\Sc \subset [s]$ with $|\Sc| = w - s$, all ones vector $(1,\ldots, 1)$ belongs to the row-space of the matrix $B_{\Sc}$.}

Our computing method crucially differs from both of the  schemes of \cite{TLDK17} and  \cite{LLPPR15}. Instead of encoding the matrices $X$ and $X^T$ with MDS codes we use a single code to encode the matrix $X^TX$, the second moment of the data.


\section{Encoding second moment~:~Optimization with approximate gradient}
\label{sec:EOP}

We exploit the special structure of the gradient 
of the square loss (cf.~\eqref{eq:gradient}) to devise a scheme to deal with stragglers. The proposed scheme is more efficient as compared to the gradient coding approach~\cite{TLDK17}  and the reliable distributed matrix multiplication based scheme~\cite{LLPPR15} (cf.~supplementary material). 
Recall the gradient of the total empirical loss associated with the square loss function from \eqref{eq:gradient}. 
Note that we need to compute the term $X^T\yv$ only once at the beginning of the optimization procedure as it is independent of the optimization parameter $\bm{\theta}$. 
By using the notation $M = X^TX$  and $\bv = X^T\yv$, 
the $t$-th step of the PGD method takes the following form (cf.~\eqref{eq:PGD_iter}). 
\begin{align}
\label{eq:gradient_des}
\bm{\theta}_{t} &= P_{\Theta}\big( \bm{\theta}_{t-1} - \eta_{t} \nabla_{\bm{\theta}}\Lc(\bm{\theta}_{t-1})\big) = P_{\Theta}\big(\bm{\theta}_{t-1} - \eta_{t}(M\bm{\theta}_{t-1} - \bv)\big).
\end{align}
where $\bm{\theta}_{t}$ denotes the estimate of $\bm{\theta}^{\ast}$ at the end of $t$-th step. 

\subsection{Exact computation of gradient in each step}
\label{sec:exact_moment}

Now, in order to perform the projected gradient descent  in a distributed computation setup, we distribute the task of computing matrix-vector product $M\bm{\theta}_t$ among the $w$ workers. In particular, we encode the $k \times k$ matrix $M$ using a linear code. The encoded matrix can be used to generate redundant  tasks for workers which subsequently enable us to mitigate the effect of stragglers.


\begin{scheme}[{Exact gradient computation using linear codes}:] 
\label{sch:exact}
Given the matrix $M = X^TX$ and an $({N = w}, K)$ linear code\footnote{For the ease of exposition, we assume that  $K$ divides $k$.} $C$, the gradient computation for each step of the optimization procedure is realized as below.
\begin{itemize}[noitemsep,topsep=0pt,parsep=0pt,partopsep=0pt]
\item Let $\mv_1, \ldots, \mv_{k}$ denote the $k$ rows of the matrix $M = X^TX$. Let $\Pc_1, \Pc_2,\ldots, \Pc_{{k}/{K}} \subset [k]$ represent a partition of the set indices for these rows $[k]$ such that 
$\Pc_i \cap \Pc_j = \emptyset~~\text{for}~i \neq j~\text{and}~|\Pc_i| = K~~\forall~i \in [{k}/{K}].$
\item For each $i \in [{k}/{K}]$, we encode the $K \times k$ matrix $M_{\Pc_i}$ using the $(N = w, K)$ linear code $\Cc$ as
$C^{(i)} = G M_{\Pc_i}  \in \R^{N \times k},$
where $G$ is an $N \times K$ generator matrix of $\Cc$. Note that the $k$ columns of the matrix $C^{(i)}$ form $k$ codewords of $\Cc$.
\item In the distributed computation setup, for $i \in [{k}/{K}]$ and $j \in [N] = [w]$, we now allocate $j$-th row of $C^{(i)}$ to the $j$-th worker. This way, the $j$-th server is assigned the following sets of $\alpha = {\frac{N}{w}}\cdot\frac{k}{K} = \frac{k}{K}$ vectors. 
\begin{align}
\label{eq:taskj}
\Tc_{j} = \big\{\cv^{(1)}_{j},\ldots, \cv^{(\frac{k}{K})}_{j}\big\} \subset \R^{k},
\end{align}
where $\cv^{(i)}_{j}$ denotes the $j$-th row of the matrix $C^{(i)}$.
\item During the $t$-th step of the gradient descent optimization procedure, $j$-th worker is tasked with computing the inner product of the rows assigned to it with the current estimate $\theta_{t-1}$, i.e., the $j$-th worker sends $\alpha = \frac{k}{K}$ inner products
$\big\{\langle\cv^{(1)}_{j},\bm{\theta}_{t-1}\rangle,\ldots, \langle\cv^{(\frac{k}{K})}_{j},\bm{\theta}_{t-1}\rangle\big\} $
 to the master.
\item \textbf{Straggler tolerant exact gradient computation:~}Assuming that the workers indexed by the set $\Sc_{t}^C := [w]\backslash \Sc_{t}$ behave as stragglers during the $t$-th step of the optimization procedure, the master has access to the following information received from the non-straggling workers. 
\begin{align}
\label{eq:info_master}
C^{(i)}_{\Sc_{t}}\bm{\theta}_{t-1} = G_{\Sc_{t}}M_{\Pc_i}\bm{\theta}_{t-1}~~\text{for all}~i \in [{k}/{K}].
\end{align}
Since the code $\Cc$ generated by $G$ is a linear code, it's straightforward to verify that for each $i \in [{k}/{K}]$, $C^{(i)}\bm{\theta}_{t-1} = GM_{\Pc_i}\bm{\theta}_{t-1}$ corresponds to a codeword of $\Cc$. Moreover, the information available at the master (cf.~\eqref{eq:info_master}) is equivalent to observing these codewords with some of their coordinates erased. Assuming the code $\Cc$ has large enough minimum distance, or equivalently, the matrix $G_{\Sc_{t}}$ is full rank, the master can recover $M_{\Pc_1}\bm{\theta}_{t-1},\ldots, M_{\Pc_{k/K}}\bm{\theta}_{t-1}$ from the information received from the workers indexed by the set $\Sc_{t}$. This allows the master to construct $M\bm{\theta}_{t-1} = X^TX\bm{\theta}_{t-1}$ and update the estimate for $\bm{\theta}$ according to \eqref{eq:gradient_des}.
\end{itemize}
\end{scheme} 

We now state the following result about the performance of Scheme~\ref{sch:exact}, which follows from the description of the scheme in a straightforward manner.

\begin{proposition}
\label{prop:exact}
Assume that the moment encoding based Scheme~\ref{sch:exact} employs an $(N = w, K)$ linear code $\Cc$ with minimum distance $d_{\min}$. Then, the scheme implements exact gradient descent method as long as the number of the stragglers during each step of the optimization  is strictly less than $d_{\min}$.  
\end{proposition}
  
\begin{remark}
Note that length of the code $\Cc$ does not need to be equal to the number of workers. For the ease of exposition, we focus on the $N = w$ case here. This choice provides a simple natural allocation of computation tasks to the workers. However, suitable allocation can also be devised for the setting with $N \neq w$.
\end{remark}

\noindent{\bf Comparison with gradient coding approach~\cite{TLDK17}.}
Encoding the second moments offers an immediate advantage over the general gradient coding approach for the underlying optimization problem (cf.~\eqref{eq:gLasso}). In Scheme~1, during a step of the optimization method, each worker communicates one scalar for each of the rows assigned to it. 
Whereas, in gradient coding, each worker needs to transmit a $k$-dimensional vector to the master (cf.~\eqref{eq:tx_gradient}, in the supplementary material). Moreover, as for the local computation at a worker during each step, our approach requires computing a single inner product for every row assigned to the worker. 
In contrast, in the gradient coding framework, workers have to perform matrix-vector products between $k \times k$ rank $1$ matrices and $k$-dimensional vectors. 

In Scheme~\ref{sch:exact}, we employ linear codes with the objective that the master should be able to compute (decode) the exact gradient 
during every step of the optimization procedure. This can be achieved by utilizing any linear code with large enough minimum distance. However, for the PGD method to succeed, it's not necessary to compute the exact gradient in every step. In particular, the stochastic gradient descent method is one of the most used versions of the gradient descent methods, where one employs an estimate of the gradient based on a randomly chosen sample and its label~\cite{UnderstandingML}. For the problem at hand, the $t$-th step of projected stochastic gradient descent (PSGD) method is as follows. 
\begin{align}
\label{eq:sgd_iteration}
\bm{\theta}_{t} = P_{\Theta}\big(\bm{\theta}_{t-1} - {\eta_{t}}\cdot m\cdot (\xv_i \xv_i^T\bm{\theta}_{t-1} - y_i\xv_i )\big),
\end{align}
where $i$ denotes an integer that is picked uniformly at random from $[m]$. Note that $m\cdot\big(\xv_i \xv_i^T\bm{\theta} - y_i\xv_i\big)$ indeed gives an unbiased estimate of the true gradient (cf.~\eqref{eq:gradient}) as 
$m\cdot \mathbb{E}\big[\big(\xv_i \xv_i^T\bm{\theta} - y_i\xv_i\big)\big] 
= \nabla_{\bm{\theta}}\Lc(\bm{\theta}).$
Next, we  exploit this robustness of the gradient based procedures to the quality of the gradient.

\subsection{Approximate recovery of gradient in every step}
\label{sec:approximate}

Here, we focus on implementing the gradient based optimization procedure in a distributed computing setup by constructing only an estimate of the true gradient during each step of the optimization procedure. This allows us to employ coding schemes that have low complexity encoding and decoding algorithms, which lowers the overall computational complexity of the coding based approach to mitigate the effect of stragglers. Before we describe our approximate gradient based optimization procedure, we specify the assumptions on the identity of the stragglers during each step of the optimization procedure. 
\begin{assumption}[{Straggling behavior of the workers}]
\label{assum:straggler}
Let the indices of the stragglers $\Sc^{C}_{t} \subset [w]$ during the $t$-th step of the optimization be distributed independent of the stragglers in the previous steps. Furthermore, let the distribution of the stragglers in each step be such that each worker independently behaves as a straggler with probability $q_0$. 
\end{assumption}
The analysis of this section can be modified for the other random models for the identity of the stragglers. Here, we note that we do not ensure any such random model for the straggling behavior during our experimental evaluations of the proposed scheme in Section~\ref{sec:simulations}.

We are now in the position to describe the LDPC codes based optimization procedure that rely on approximate gradient during each step of the optimization procedure.

\begin{scheme}[{LDPC codes based optimization with approximate gradients}] 
\label{sch:ldlc}
Given the matrix $M = X^TX$, we take an $(N = w = k + p, K = k)$ LDPC code $\Cc$ with $H \in \R^{p \times N}$ as its (low-density) parity check matrix.\footnote{For the ease of exposition, in addition to $N = w$,  we assume that $K = k$. The proposed scheme can be easily generalized to the setting with $k > K$, as done in Scheme~\ref{sch:exact} by partitioning the rows of $M$ in the blocks of $K$ rows.} The approximate gradient based optimization procedure in realized as follows.
\begin{itemize}[noitemsep,topsep=0pt,parsep=0pt,partopsep=0pt]
\item Encode $M = X^TX$ using a systematic matrix of $\Cc$, say $G$, as 
$C = GM,$
where without loss of generality we assume that $M$ constitutes the first $k$ rows of the matrix $C$. Next, distribute the $w = k + p$ rows of $C$ among $w$ workers such that the $j$-th row $\cv_j$ is assigned to the $j$-th worker.
\item During the $t$-th step of the optimization procedure, $j$-th worker computes the inner product of the row assigned to it with the current estimate $\theta_{t-1}$ and sends $\cv^{(1)}_{j}\theta_{t-1} \in \R$ to the master.
\item Assuming that the set $\Sc_{t}^C := [w]\backslash \Sc_{t}$ denotes the indices of stragglers during $t$-th step, the information received at the master takes the  form:
\begin{align}
\label{eq:info_master1}
C_{\Sc_{t}}\bm{\theta}_{t-1} = G_{\Sc_{t}}M\bm{\theta}_{t-1}.
\end{align}
Note that $\cv = GM\bm{\theta}_{t-1}$ is a codeword of $\Cc$ with $M\bm{\theta}_{t-1}$ appearing in its first $k$ coordinates. 
\item \textbf{Computation of approximate gradient:~}Given $\cv_{\Sc_{t}} = C_{\Sc_{t}}\bm{\theta}_{t-1} = G_{\Sc_{t}}M\bm{\theta}_{t-1}$, the master employs $D$ iterations of an iterative erasure correction algorithm for the LDPC code $\Cc$, where $\Sc_{t}^{c}$ denotes the indices of the erased coordinates. Let $\hat{\cv}(t;D) = (\hat{c}(t;D)_1,\ldots, \hat{c}(t;D)_k)$ be the estimate for the codeword $\cv$ after $D$ iterations of the erasure correction algorithm~\cite{MCT08}. If a particular coordinate is not recovered by the end of $D$ iterations, we replace the coordinate with $0$. Let ${\Uc_{t}} \subseteq [k]$ denote the set of indices of the coordinates that are set to $0$ in this manner. Subsequently, we construct a vector $\widehat{\bv}_{t}$ by setting those coordinates of $\bv = X^T\yv$ to $0$ that are in $\Uc_{t}$. During the $t$-th step, the master updates the current estimate of $\bm{\theta}^{\ast}$ as
\begin{align}
\label{eq:GradientDescent_new}
\bm{\theta}_{t} = P_{\Theta}\Big( \bm{\theta}_{t-1} - \eta_l \cdot \left(\big({\hat{c}(t;D)_1},\ldots,{\hat{c}(t;D)_k}\big)^T - \widehat{\bv}_t \right)\Big).
\end{align}
\end{itemize}
\end{scheme}

In what follows, we establish that under Assumption~\ref{assum:straggler}, Scheme~\ref{sch:ldlc} indeed implements a variant of the PSGD method. As a result, under some natural requirements on the loss function and the initialization $\bm{\theta}_0$, we obtain a convergence result for Scheme~\ref{sch:ldlc} that is similar to those available in the literature for the PSGD method (cf.~\eqref{eq:sgd_iteration}). 

However, before we analyze the convergence of Scheme~\ref{sch:ldlc}, we need to characterize the quality of the gradient recovered at the end of $D$ iterations of the erasure correction algorithm of the underlying LDPC code $\Cc$. The LDPC codes have been extensively studied in the literature along with the performances of various decoding algorithms for such codes~\cite{gallager,SS96,MCT08}. Under Assumption~\ref{assum:straggler}, where each worker independently behaves as a straggler with probability $q_0$, the vector received by the master (cf.~\eqref{eq:info_master1}) is equivalent to the outcome of an erasure channel. For a specific family of LDPC codes and a fixed iterative erasure correction algorithm, there have been many successful attempts to characterize the likelihood of an initially erased coordinate being recovered after a certain number of iterations. Here, we state a special case\footnote{In particular, we restrict ourselves to the LDPC codes with left and right regular Tanner graphs. We refer the readers to \cite{MCT08} for the general version of the result that applies to LDPC codes with irregular Tanner graphs.} of the most prominent result in this direction which applies to various random ensembles of LDPC codes with sufficiently large length. This results is obtained by density evolution analysis~\cite{MCT08}.

\begin{proposition}
\label{prop:de}
Consider an ensemble of LDPC code defined by the random $p \times N$ parity check matrix $H$ such that each of the $p$ rows ($N$ columns) of the matrix $H$ have $l$ ($r$) nonzero entries.\footnote{There are multiple ways of generating a  random ensembles of LDPC codes (see e.g., \cite{MCT08}[Ch.~3]).} Let each coordinate of a codeword from the ensemble be independently erased with the probability $q_0$. Then, the probability $q_d$ that a coordinate of the codeword remains erased after $d$ iterations of the iterative erasure correction satisfies the  relationship\footnote{The relation in here is shown to hold with very high probability, which involves application of bounded-difference concentration inequality on the random bipartite graphs corresponding to $H$. Given that these are fairly standard results in the coding theory literature, we refer the readers to \cite{RU01, MCT08} for the details.}
$q_{d} = q_{0}\cdot\big(1 - (1 - q_{d-1})^{r -1}\big)^{l - 1}.
$
\end{proposition}

\begin{remark}
The key take away from Proposition~\ref{prop:de} is that the probability of a coordinate of a codeword staying erased is a {\em monotonically non-increasing} function of the number of iterations as long as $q_0 < q^{\ast}(r,l) < 1$, where $q^{\ast}(r,l)$ is function of the row and column weights of the random matrix $H$. 
\end{remark}

 The following lemma characterizes the quality of the gradient vector obtained at the master after $D$ iterations of the erasure correction algorithm of the underlying LDPC code.

\begin{lemma}
\label{lem:unbiased}
Let the distribution of stragglers satisfy Assumption~\ref{assum:straggler} and the master node employs $D$ iterations of the erasure correction algorithm. Then, during $t$-th step of the optimization procedure, we have
$$
\mathbb{E}\left[\big({\hat{c}(t;D)_1},\ldots,{\hat{c}(t;D)_k}\big)^T - \widehat{\bv}_t \right] = (1 - q_D) \cdot \nabla\Lc(\bm{\theta_{t-1}}),
$$
which is a scaled version of the true gradient at $\bm{\theta}_{l-1}$.
\end{lemma}

\begin{proof}
Recall that, during the $t$-th step of the optimization procedure, $q_D$ denotes the probability that a particular coordinate of the codeword $\cv = C\theta_{t-1} \in \R^{N}$ is not recovered by the master (cf.~Scheme~\ref{sch:ldlc}). The first $k$ coordinates of this vector correspond to the true gradient vector at $\theta_{t-1}$. Therefore, for $i \in [k]$, we have 
\begin{align}
\label{eq:ceq}
&\mathbb{P}\left[\hat{c}(t;D)_i = c_i \right] = 1 - q_{D}~\text{and}~\mathbb{P}\left[\hat{c}(t;D)_i = 0\right] = q_{D}.
\end{align}
Similarly, for $i \in [k]$, we have,
\begin{align}
\label{eq:beq}
\mathbb{P}\left[\hat{b}_i = b_i \right] = 1 - q_{D}~~\text{and}~~\mathbb{P}\left[\hat{b}_i = 0\right] = q_{D}.
\end{align}
By using \eqref{eq:ceq} and \eqref{eq:beq}, it is straightforward to verify that
\begin{align}
\mathbb{E}\left[\big({\hat{c}(t;D)_1},\ldots,{\hat{c}(t;D)_k}\big)^T - \widehat{\bv}_t \right] &= (1 - q_D)\cdot\left( (c_1, \ldots, c_k)^T - \bv  \right) \nonumber \\ 
&\overset{(i)}{=} (1 - q_D)\cdot\left(M\bm{\theta}_{l-1} - X^T\yv\right) \nonumber \\
&= (1 - q_D)\cdot \nabla \Lc(\bm{\theta_{t-1}}), \nonumber
\end{align}
where $(i)$ follows from the systematic form associated with the generator matrix $G$.
\end{proof}

\paragraph{Convergence analysis of Scheme~\ref{sch:ldlc}}
Here, we formally argue that the proposed Scheme~\ref{sch:ldlc} enjoys the convergence guarantees similar to those available for the typical PSGD method. In fact, the proof of the convergence of our scheme heavily relies on the ideas employed in the proof of convergence for PSGD algorithm as described in \cite{Nemi09}. 
Recall that the total empirical loss associated with the model parameter $\bm{\theta} \in \R^{n}$ for given set of data samples $\{\xv_i\}_{i \in [m]} \subset \R^{n}$ and the corresponding labels  $\{y_i\}_{i \in [m]} \subset \R$ takes the form.
\begin{align}
\Lc(\bm{\theta}) = \sum_{i = 1}^{m} \ell \big((y_i,\xv_i), \bm{\theta}\big) = \frac{1}{2}\cdot \sum_{i = 1}^{m}\big(y_i - \xv_i^T\bm{\theta}\big)^2.
\end{align}
We now state the convergence result for Scheme~\ref{sch:ldlc} which holds under natural assumptions on the loss function and the initialization for the optimization procedure $\bm{\theta}_{0}$.In what follows we use $\|\cdot\|$ to denote the $\ell_2$ norm $\|\cdot\|_2$. We also note that the projection operator $P_\Theta$ is non-expanding, i.e., 
\begin{align*}\|P_{\Theta}\big(\bm{\theta}\big) - P_{\Theta}\big(\bm{\theta}'\big)\| \leq \|\bm{\theta} - \bm{\theta}'\|~~\text{for all}~\bm{\theta}, \bm{\theta}' \in \R^{k}.\end{align*}
\begin{theorem}\label{thm:main}
Suppose for all $(\xv, y) \in \R^{k + 1}$ and $\bm{\theta} \in \Theta$, the loss function satisfies
$\|\nabla \Lc(\bm{\theta})\| \leq B.$
Moreover, let the initial estimate $\bm{\theta}_{0}$ satisfy 
$\|\bm{\theta}_{0} - \bm{\theta}^{\ast}\| \leq R.$
Then, by setting the learning rate as $\eta = {R}/({B}\cdot\sqrt{{T}})$ in Scheme~\ref{sch:ldlc}, when $D$ iterations of LDPC decoding are employed during each gradient descent step, ensures the following:
\begin{align}
\avg\big[\Lc(\bar{\bm{\theta}}_T)\big] - \Lc(\bm{\theta}^{\ast}) \leq {RB}/{\big((1 - q_{D})\cdot\sqrt{T}\big)}, 
\end{align}
where $\bar{\bm{\theta}}_{T} = \frac{1}{T}\cdot\sum_{t \in [T]}\bm{\theta}_{t}$ and the expectation is taken over the distribution of the stragglers.
\end{theorem}

\begin{proof} It follows from the convexity of the loss function $\Lc(\cdot)$ that 
\begin{align}
\label{eq:convexity}
\Lc(\bar{\bm{\theta}}_{T}) - \Lc(\bm{\theta}^{\ast})  \leq \frac{1}{T}\sum_{t = 1}^{T}\Lc(\bm{\theta}_{t}) - \Lc(\bm{\theta}^{\ast}) \leq \frac{1}{T}\sum_{t = 1}^{T}\nabla \Lc(\bm{\theta}_{t})\cdot(\bm{\theta}_t - \bm{\theta}^{\ast}).
\end{align}
Recall from \eqref{eq:GradientDescent_new} that, for $0 \leq t \leq T - 1$, we have 
\begin{align}
\bm{\theta}_{t+1} = P_{\Theta}\big(\bm{\theta}_t - g_t(\bm{\theta}_t)\big), \nonumber 
\end{align}
where $g_t(\bm{\theta}_t) = \big({\hat{c}(t+1;D)_1},\ldots,{\hat{c}({t+1;D})_k}\big)^T - \widehat{\bv}_{t+1} $. Now, consider 
\begin{align}
\label{eq:smoothness}
\|\bm{\theta}_{t+1} - \bm{\theta}^{\ast}\|^2 & \leq \|P_{\Theta}\big(\bm{\theta}_{t} - \gv_{t}(\bm{\theta}_{t})\big) - \bm{\theta}^{\ast}\|^2  \overset{(i)}{=} \|P_{\Theta}\big(\bm{\theta}_{t} - \gv_{t}(\bm{\theta}_{t})\big) - P_{\Theta}\big(\bm{\theta}^{\ast}\big)\|^2 \nonumber \\
&\leq \|\bm{\theta}_{t} - \gv_{t}(\bm{\theta}_{t}) - \bm{\theta}^{\ast}\|^2 \nonumber \\
&= \|\bm{\theta}_{t} - \bm{\theta}^{\ast}\|^2 - 2\eta \cdot \langle \gv_{t}(\bm{\theta}_{t}), (\bm{\theta}_{t} - \bm{\theta}^{\ast})\rangle + \eta^{2} \|\gv_{t}(\bm{\theta}_{t})\|^2 \nonumber \\
&\leq \|\bm{\theta}_{t} - \bm{\theta}^{\ast}\|^2 - 2\eta \cdot \langle \gv_{t}(\bm{\theta}_{t}), (\bm{\theta}_{t} - \bm{\theta}^{\ast})\rangle + \eta^{2} B^{2},
\end{align}
where $(i)$ follows from the fact that $\bm{\theta}^{\ast} \in \Theta$ and $(ii)$ holds as the operator $P_{\Theta}$ is non-expanding, i.e.,
$$\|P_{\Theta}\big(\bm{\theta}\big) - P_{\Theta}\big(\bm{\theta}'\big)\| \leq \|\bm{\theta} - \bm{\theta}'\|~~\text{for all}~\bm{\theta}, \bm{\theta}' \in \R^{k}.$$
Let $\Hc_{t}$ denote the history, i.e., identity of the stragglers, before the $(t+1)$-th step of the optimization procedure. Note that it follows from Lemma~\ref{lem:unbiased} that 
\begin{align}
\label{eq:unbiased1}
\avg[\gv_{t}(\bm{\theta}_{t})~\vert~\Hc_t] = (1 - q_D)\cdot \nabla \Lc(\bm{\theta}_{t}).
\end{align}
By combining \eqref{eq:smoothness} and \eqref{eq:unbiased1}, we obtain that
\begin{align}
\avg [ \|\bm{\theta}_{t+1} -\bm{\theta}^{\ast}\|^2~\vert~\Hc_{t}] \leq \|\bm{\theta}_{t} - \bm{\theta}^{\ast}\|^2 - 2\eta\cdot (1 - q_D) \cdot \langle \nabla \Lc(\bm{\theta}_{t}), (\bm{\theta}_{t} - \bm{\theta}^{\ast})\rangle +  \eta^2 B^2.
\end{align}
Now taking expectation on the both sides gives us that 
\begin{align}
\avg [ \|\bm{\theta}_{t+1} -\bm{\theta}^{\ast}\|^2] \leq \avg[\|\bm{\theta}_{t} - \bm{\theta}^{\ast}\|^2] - 2\cdot \avg[\eta\cdot (1-q_D)\cdot\langle \nabla \Lc(\bm{\theta}_{t}), (\bm{\theta}_{t} - \bm{\theta}^{\ast})\rangle] +  \eta^2 B^2.
\end{align}
or 
\begin{align}
(1 - q_D)\cdot\avg[\langle \nabla \Lc(\bm{\theta}_{t}), (\bm{\theta}_{t} - \bm{\theta}^{\ast})\rangle] & \leq \frac{1}{2\eta}\cdot\avg[\|\bm{\theta}_{t} - \bm{\theta}^{\ast}\|^2] - \frac{1}{2\eta}\cdot \avg [ \|\bm{\theta}_{t+1} -\bm{\theta}^{\ast}\|^2]  +  \frac{\eta B^2}{2}.
\end{align}
By taking the average of the aforementioned inequality over $T$ iteration, we obtain that 
\begin{align}
\label{eq:final_ineq}
\avg \left[ \frac{1}{T} \sum_{t = 0}^{T-1}   \langle \nabla \Lc(\bm{\theta}_{t}), (\bm{\theta}_{t} - \bm{\theta}^{\ast})\rangle \right]  &\leq \frac{1}{2\eta(1 - q_D)}\cdot\Big(\frac{\avg[\|\bm{\theta}_{0} - \bm{\theta}^{\ast}\|^2]}{T} - \frac{\avg [ \|\bm{\theta}_{T} -\bm{\theta}^{\ast}\|^2]}{T}  +  {\eta^2 B^2}\Big) \nonumber \\
&\leq \frac{\|\bm{\theta}_{0} - \bm{\theta}^{\ast}\|^2}{2\eta T(1- q_D)} + \frac{\eta B^2}{2(1-q_D)} \nonumber \\
&\leq \frac{R^2}{2\eta T(1- q_D)} + \frac{\eta B^2}{2(1-q_D)} \nonumber \\
&\overset{(i)}{\leq}\frac{1}{1 - q_D} \cdot \frac{R B}{\sqrt{T}},
\end{align}
where $(i)$ follows form the choice of $\eta$. Now, Theorem~\ref{thm:main} follows by combining \eqref{eq:convexity} and \eqref{eq:final_ineq}.
\end{proof}

\section{Simulation results}
\label{sec:simulations}

In this section, we conduct a detailed evaluation of our  moment encoding based scheme (cf.~Scheme \ref{sch:ldlc}) for distributed computation. In particular, we perform experiments on distributed  setting to obtain solutions of two problems: 1) Least-square estimation, and 2) Sparse recovery. Recall that, for  least-square estimation, given inputs $\yv \in \reals^m$ and $X\equiv \{x_1,x_2,\ldots x_m \} \in \reals^{m \times k}$ the task is to find 
$
\arg \min_{\bm{\theta} \in \reals^k} \|\yv - X\bm{\theta}\|_2^2.
$
Note that this problem does not require a projection step during the optimization procedure.
In the sparse recovery problem, one seeks to find a $u$-sparse vector $\bm{\theta} \in \reals^k$ (this means at most $u$ coordinates out of $k$ of the vector $\bm{\theta}$ are nonzero) from linear samples $\yv = X \bm{\theta}$, for some matrix $X \in \reals^{m \times k}$. In this case, $t$-th step of the projected gradient descent procedure takes the form~\cite{garg2009gradient}
$
\bm{\theta}_t = H_u(\bm{\theta}_{t-1} - \eta \nabla_{\bm{\theta}}\Lc(\bm{\theta}_{t-1})),
$
where $\nabla_{\bm{\theta}}\Lc(\bm{\theta}) =  X^TX\bm{\theta} - X^T\yv$ is the gradient of the squared loss $ \|\yv - X\bm{\theta}\|_2^2$ and $H_u(\wv)$ is the thresholding operation that sets all except the largest $u$ coordinates in absolute value of $\wv \in \reals^k$ to zero. To compute the gradient we again employ the moment encoding method with LDPC codes as outlined in Scheme \ref{sch:ldlc}. Note that the thresholding operation can be easily performed by the master node itself.

\setlength{\belowcaptionskip}{-10pt}
\begin{figure*}[t!]
        \centering
        \begin{subfigure}[b]{0.23\textwidth}
                \centering
                \includegraphics[width=1\textwidth]{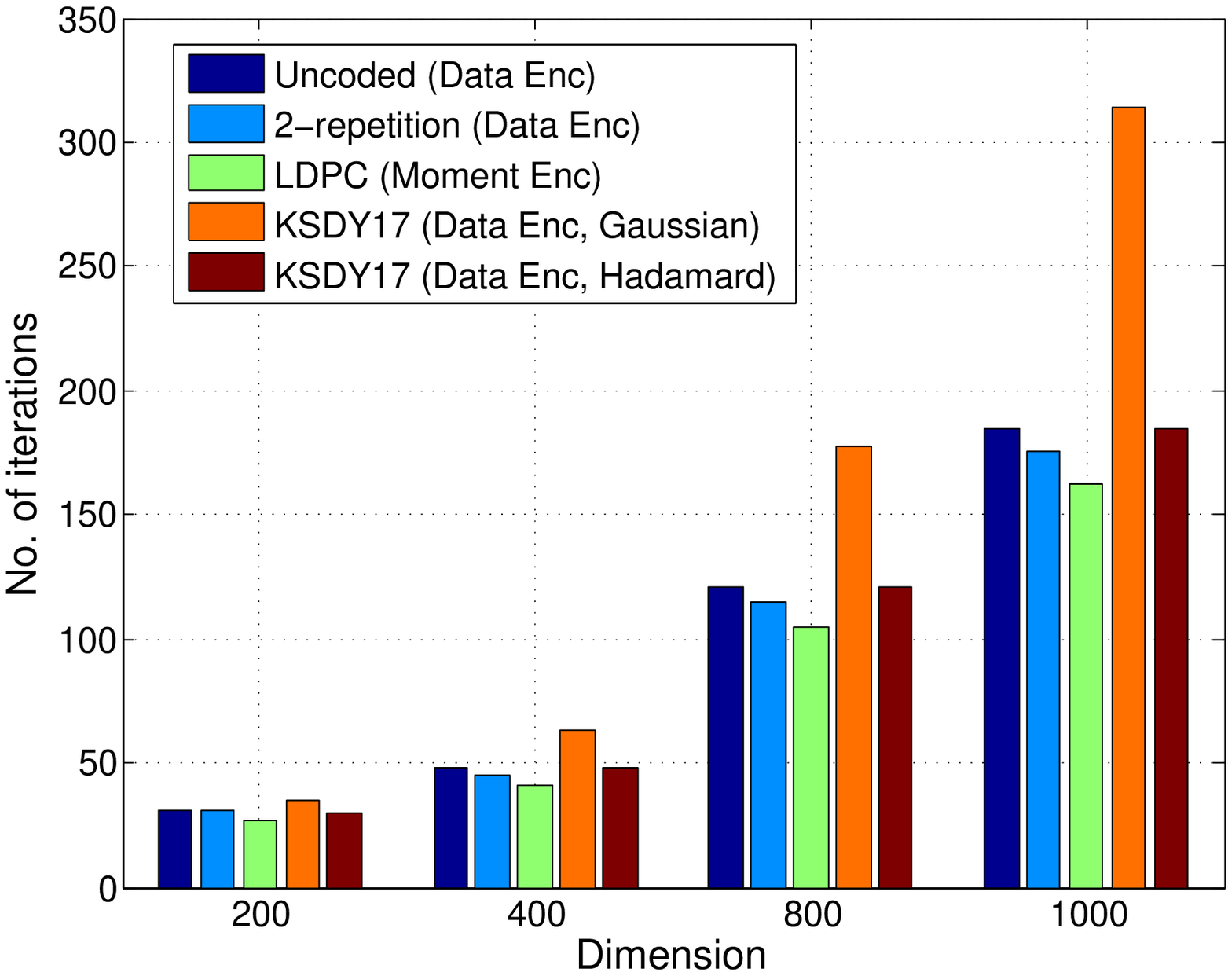}
                \label{fig:Grad35_iter}
        \end{subfigure}%
       ~~
        \begin{subfigure}[b]{0.23\textwidth}
   		\footnotesize
                \centering
                \includegraphics[width=1\textwidth]{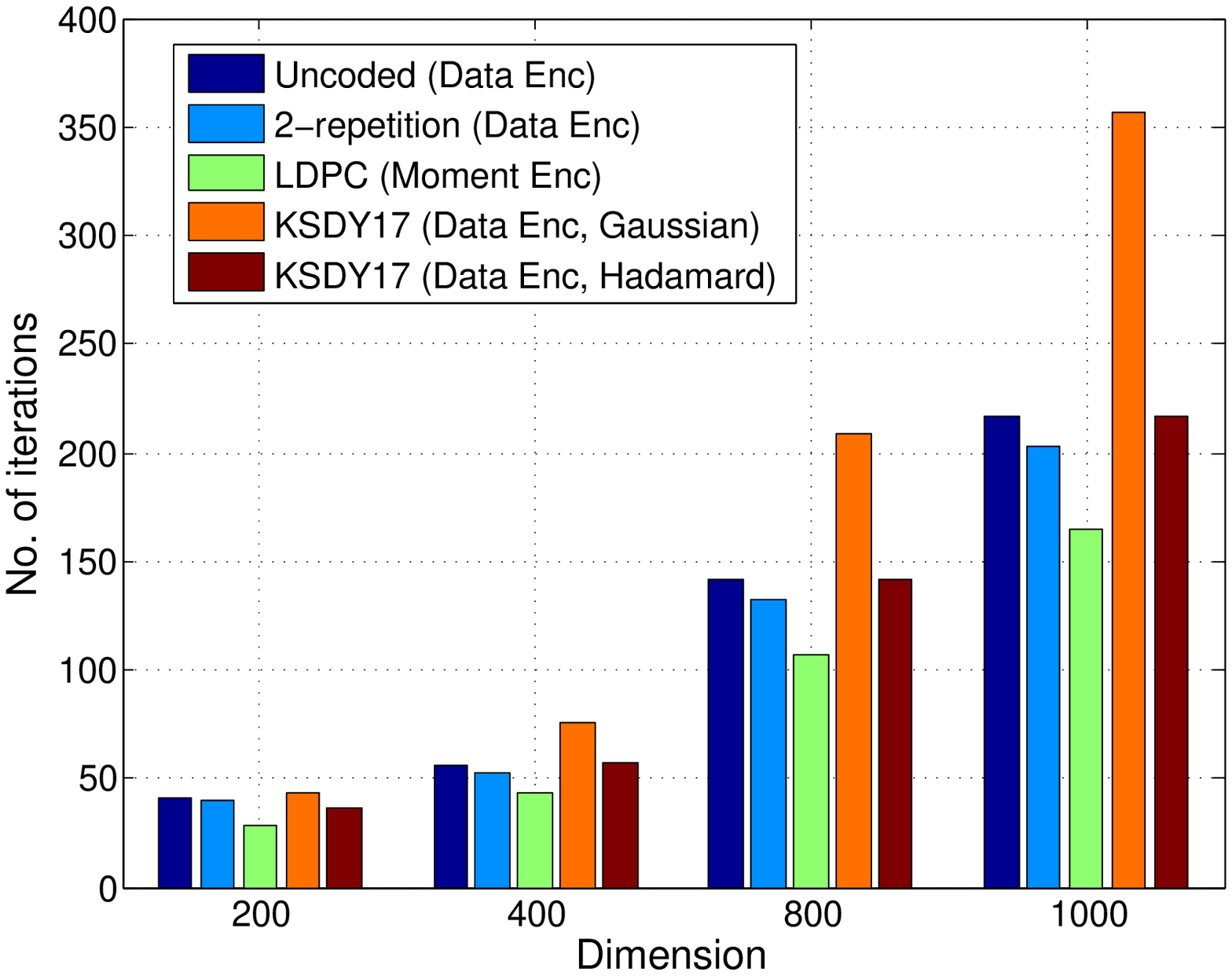}
                \label{fig:Grad30_iter}
        \end{subfigure}~~
        \begin{subfigure}[b]{0.23\textwidth}
                \centering
                \includegraphics[width=1\textwidth]{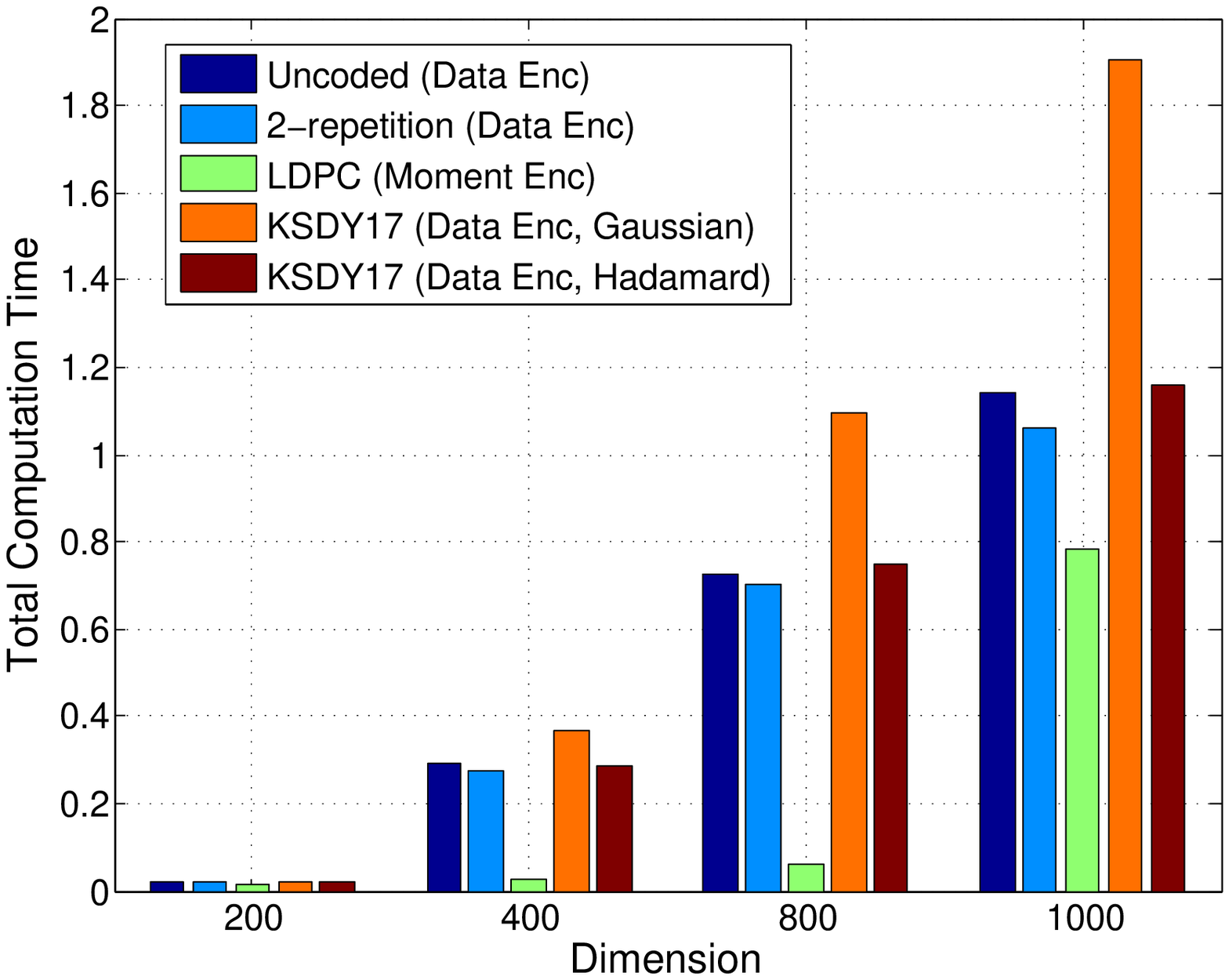}
                \label{fig:Grad35_time}
        \end{subfigure}%
       ~~
        \begin{subfigure}[b]{0.23\textwidth}
   		\footnotesize
                \centering
                \includegraphics[width=1\textwidth]{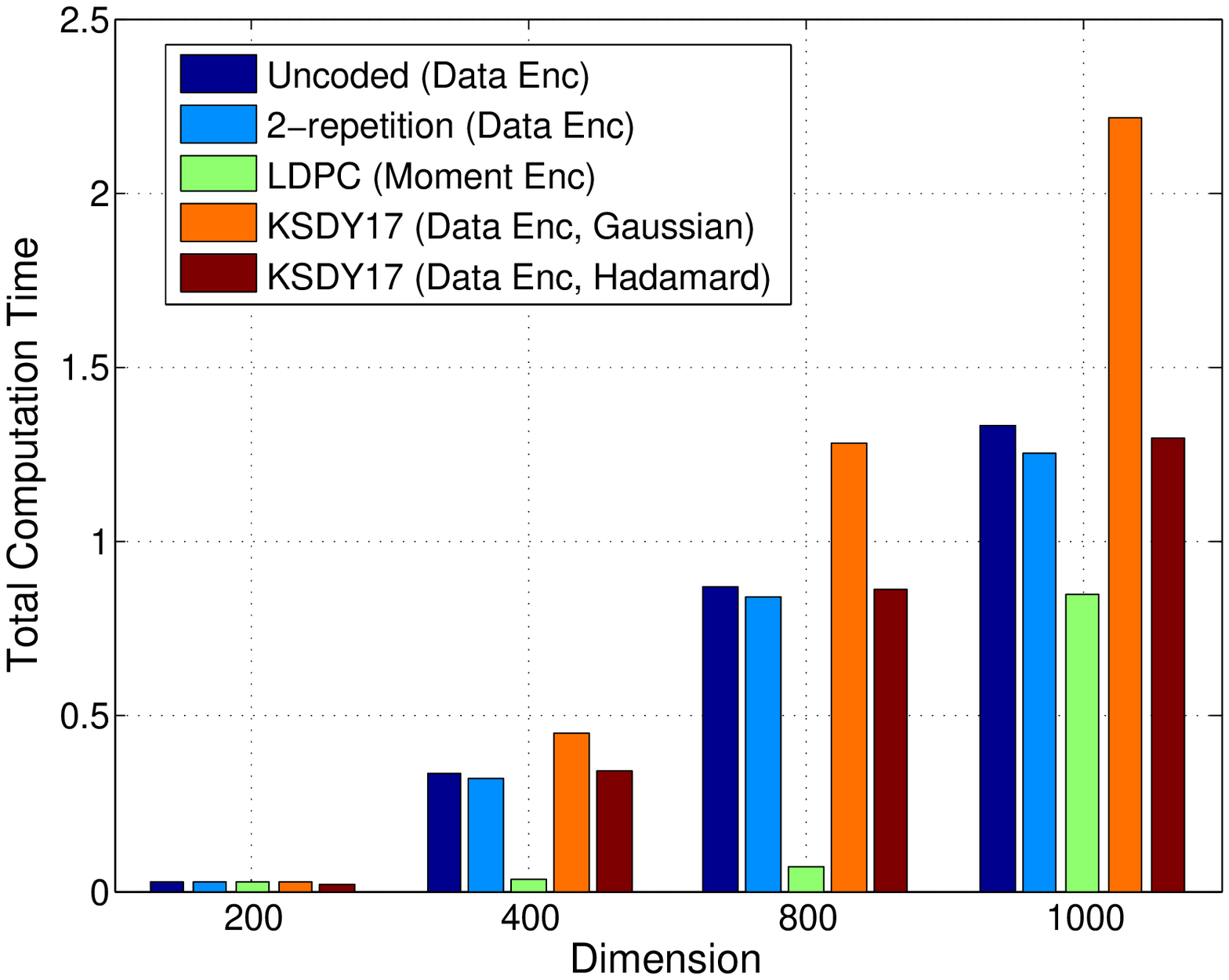}
                \label{fig:Grad30_time}
        \end{subfigure}
        \caption{\small Total number of iterations and total computation time for solving the linear regression problem $(m = 2048)$. The number of stragglers are 5, 10, 5 and 10 from left to right. \label{fig:ls}}

        \vspace*{\floatsep}

        \begin{subfigure}[b]{0.23\textwidth}
                \centering
                \includegraphics[width=1\textwidth]{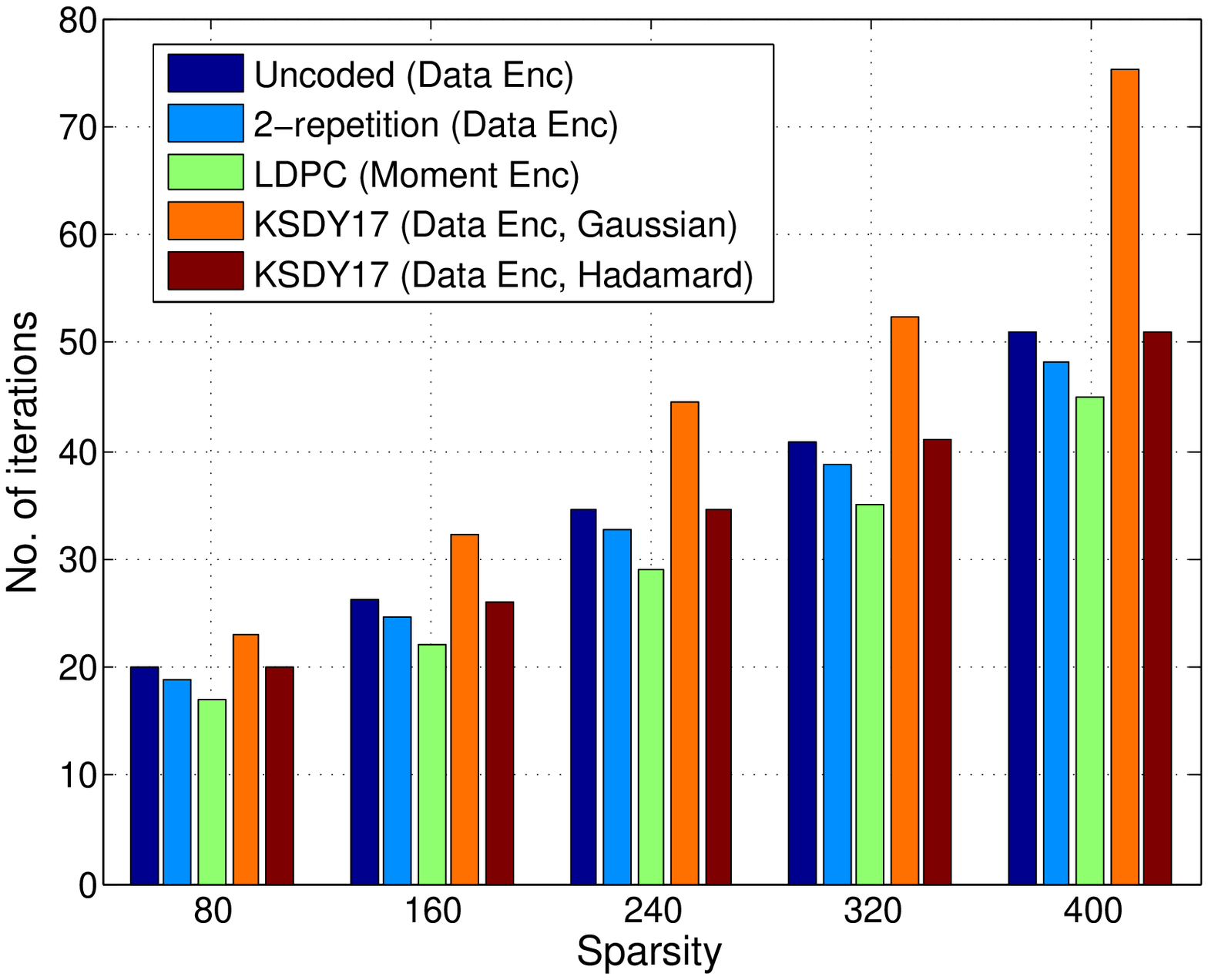}
                \label{fig:Sparse35_1000_iter}
        \end{subfigure}%
       ~~
        \begin{subfigure}[b]{0.23\textwidth}
   		\footnotesize
                \centering
                \includegraphics[width=1\textwidth]{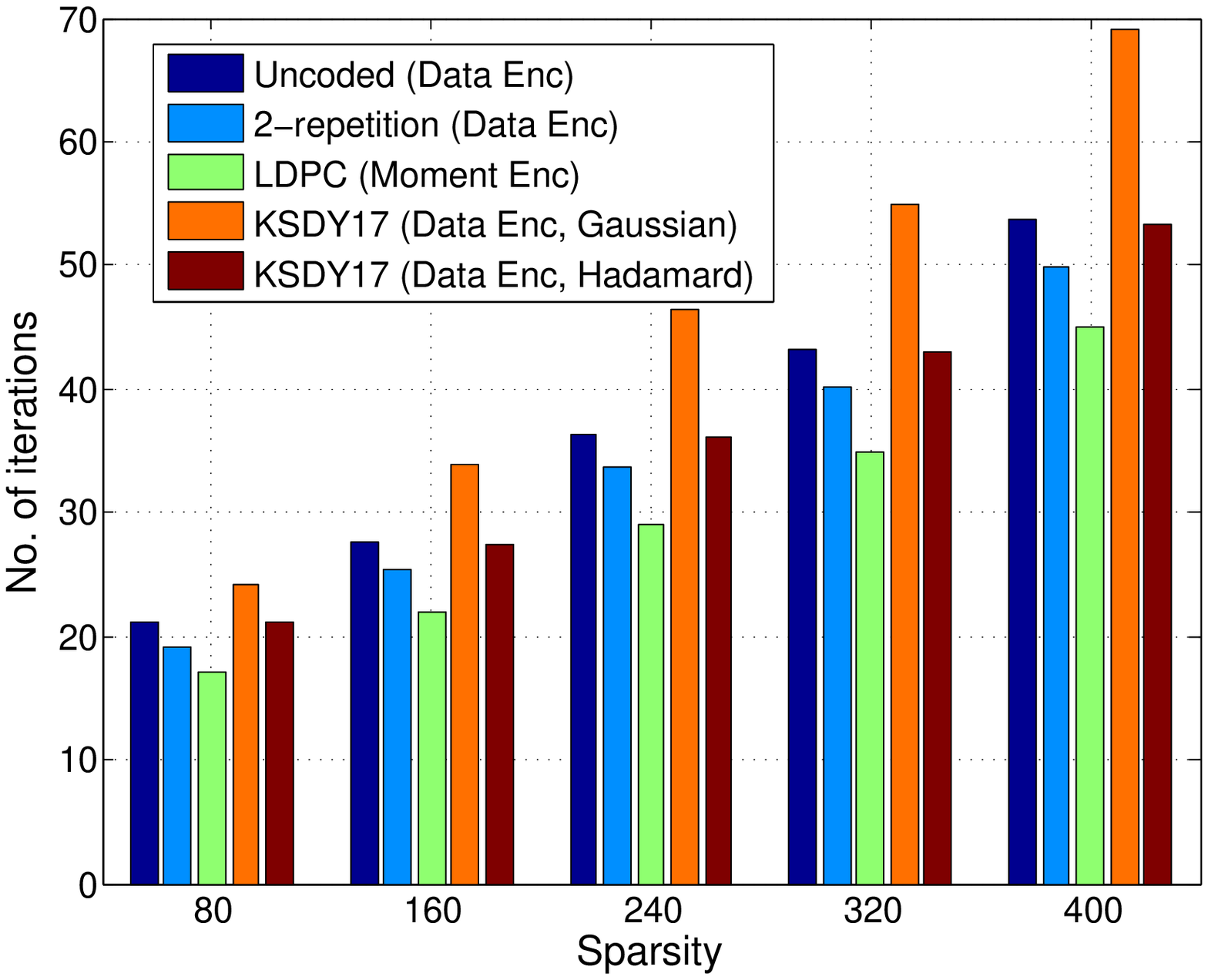}
                \label{fig:Sparse30_800_iter}
        \end{subfigure}~~
        \begin{subfigure}[b]{0.23\textwidth}
                \centering
                \includegraphics[width=1\textwidth]{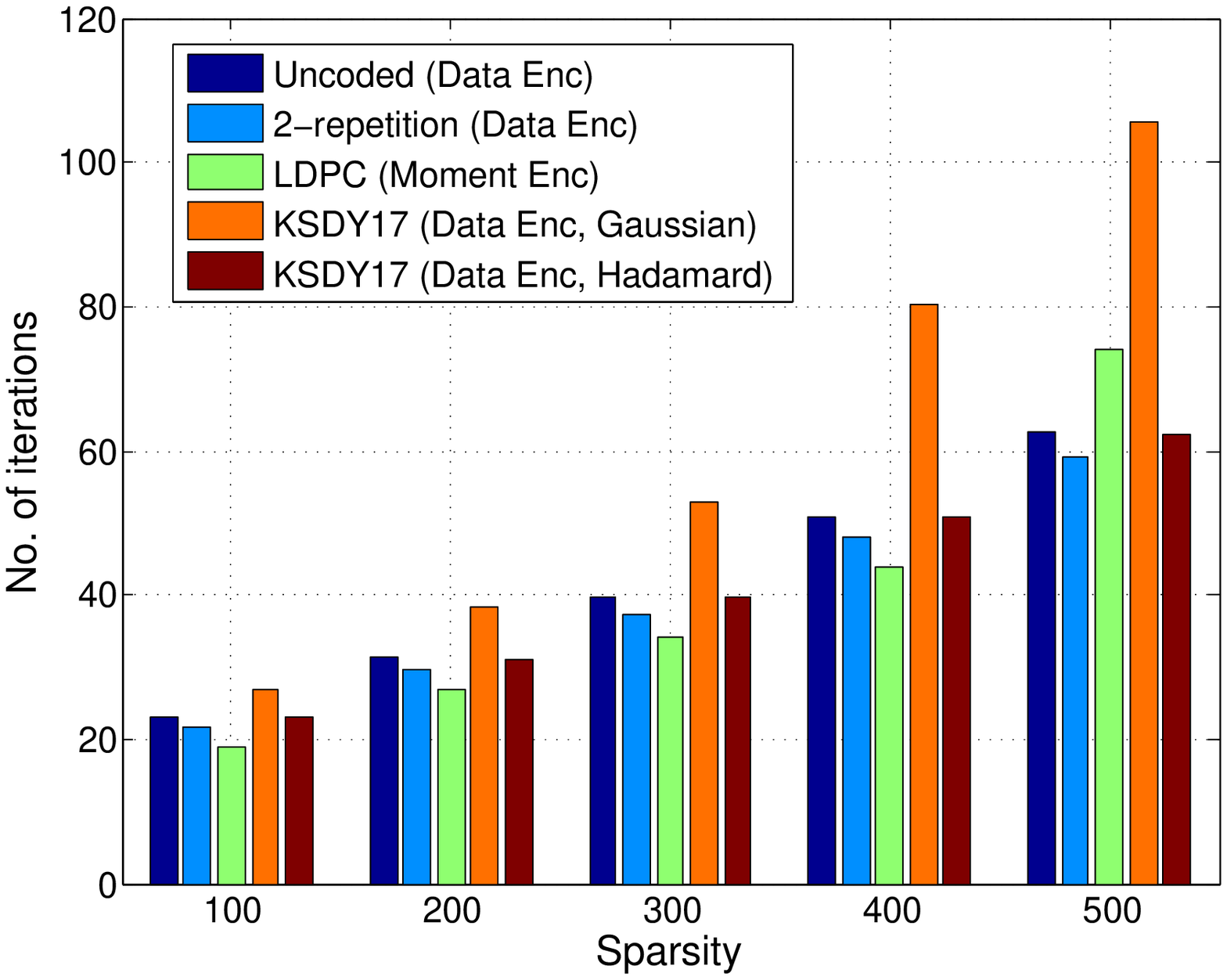}
                \label{fig:Sparse35_1000_iter}
        \end{subfigure}%
       ~~
        \begin{subfigure}[b]{0.23\textwidth}
   		\footnotesize
                \centering
                \includegraphics[width=1\textwidth]{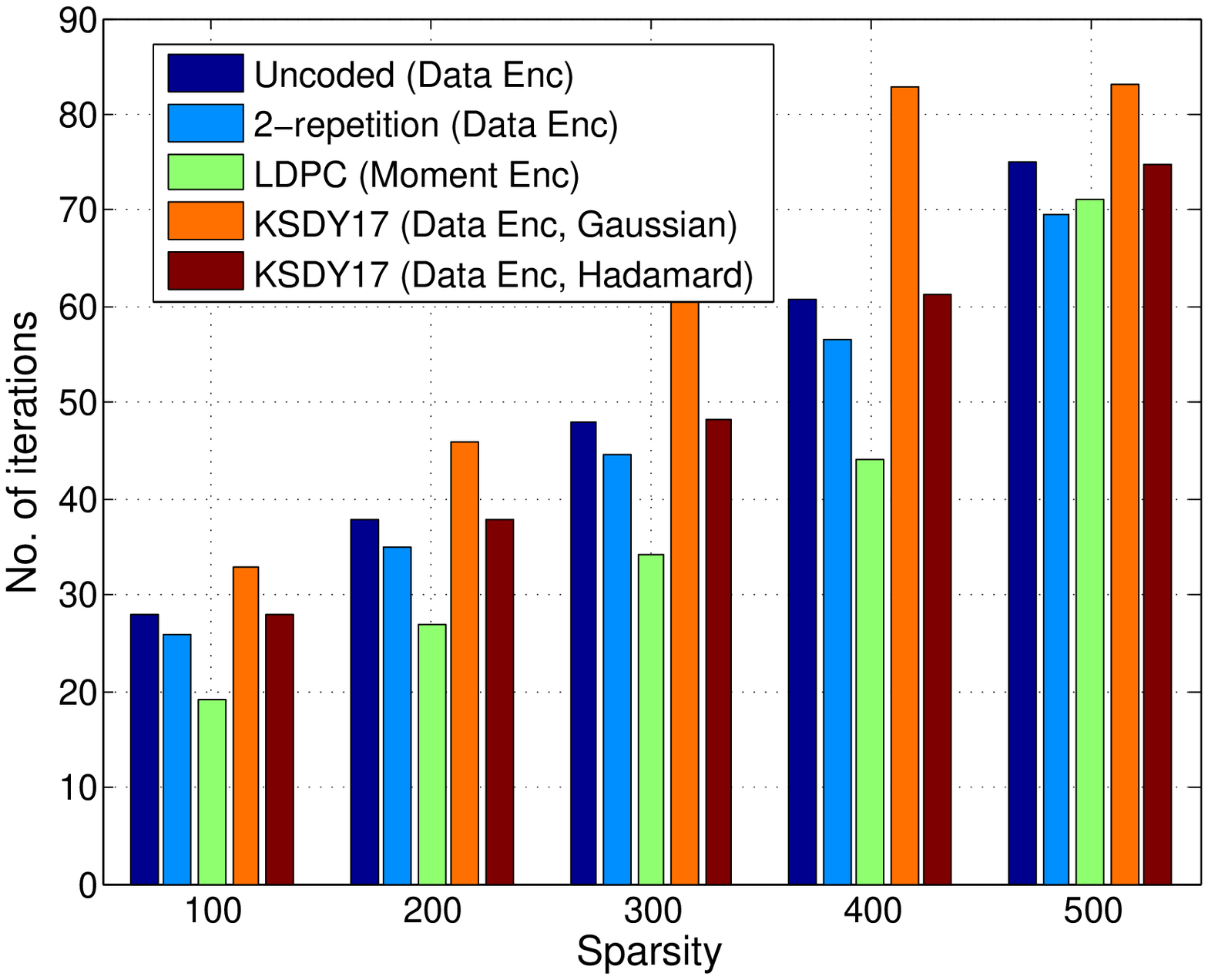}
                \label{fig:Sparse30_1000_iter}
        \end{subfigure}
        \caption{\small Total number of iterations for solving sparse recovery problem in an overdetermined system $(m = 2048)$. 
        The left two figures correspond to the dimension $800$ and the remaining ones correspond to dimension $1000$. The number of stragglers are 5, 10, 5 and 10 from left to right. \label{fig:SRover}}

        \vspace*{\floatsep}

        \begin{subfigure}[b]{0.23\textwidth}
                \centering
                \includegraphics[width=1\textwidth]{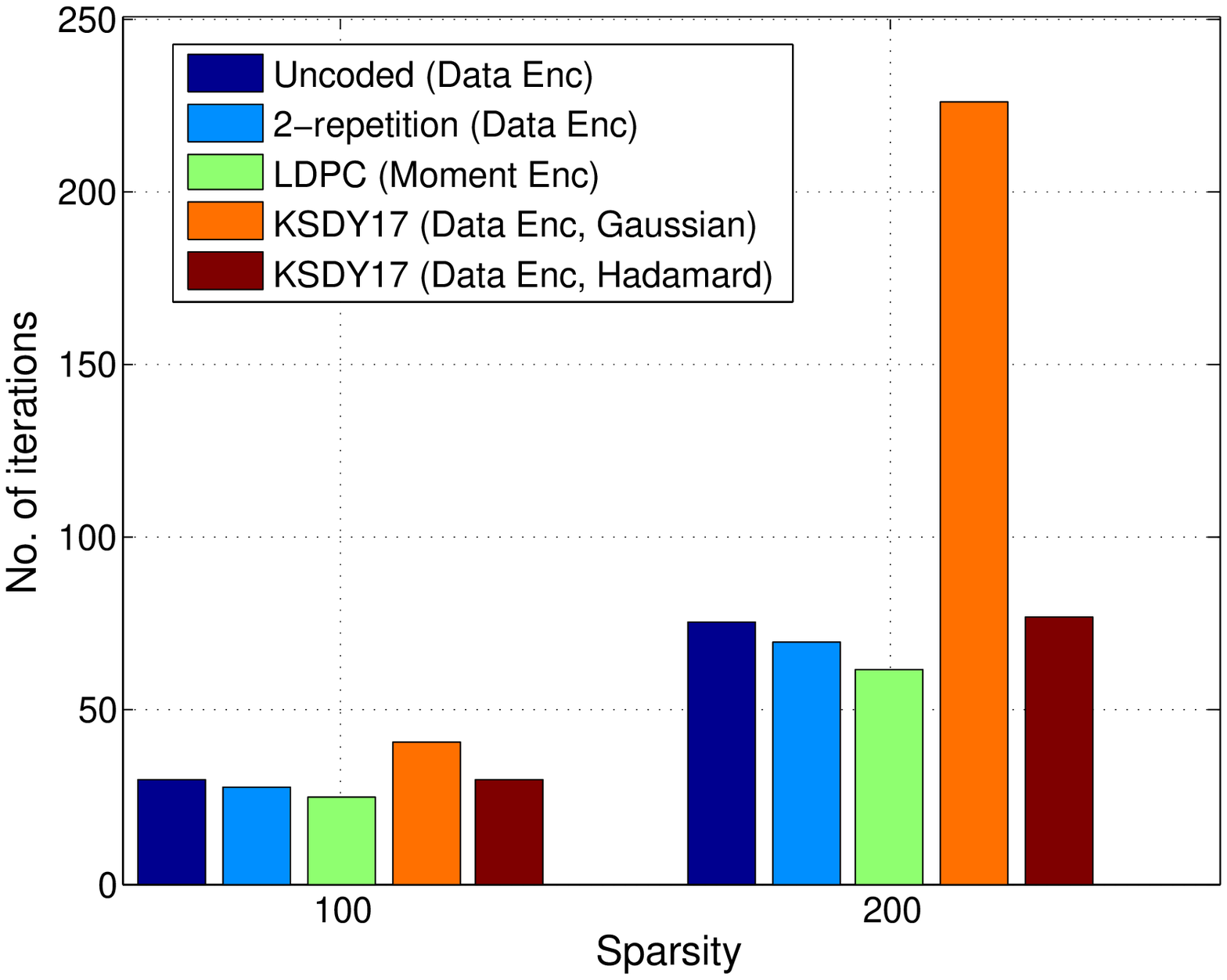}
                \label{fig:HDSparse35_iter}
        \end{subfigure}%
       ~~
        \begin{subfigure}[b]{0.23\textwidth}
   		\footnotesize
                \centering
                \includegraphics[width=1\textwidth]{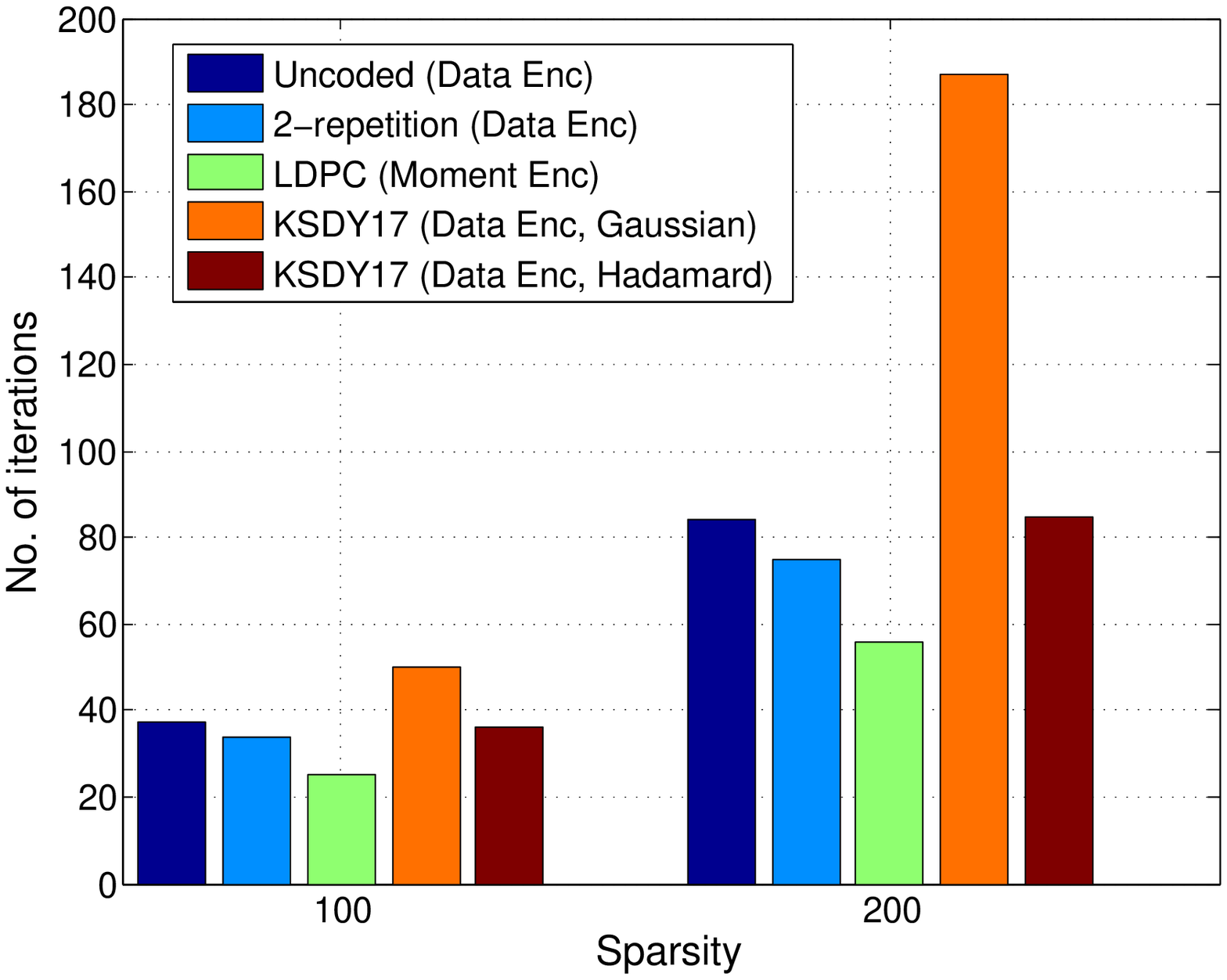}
                \label{fig:HDSparse30_iter}
        \end{subfigure}~~
        \begin{subfigure}[b]{0.23\textwidth}
                \centering
                \includegraphics[width=1\textwidth]{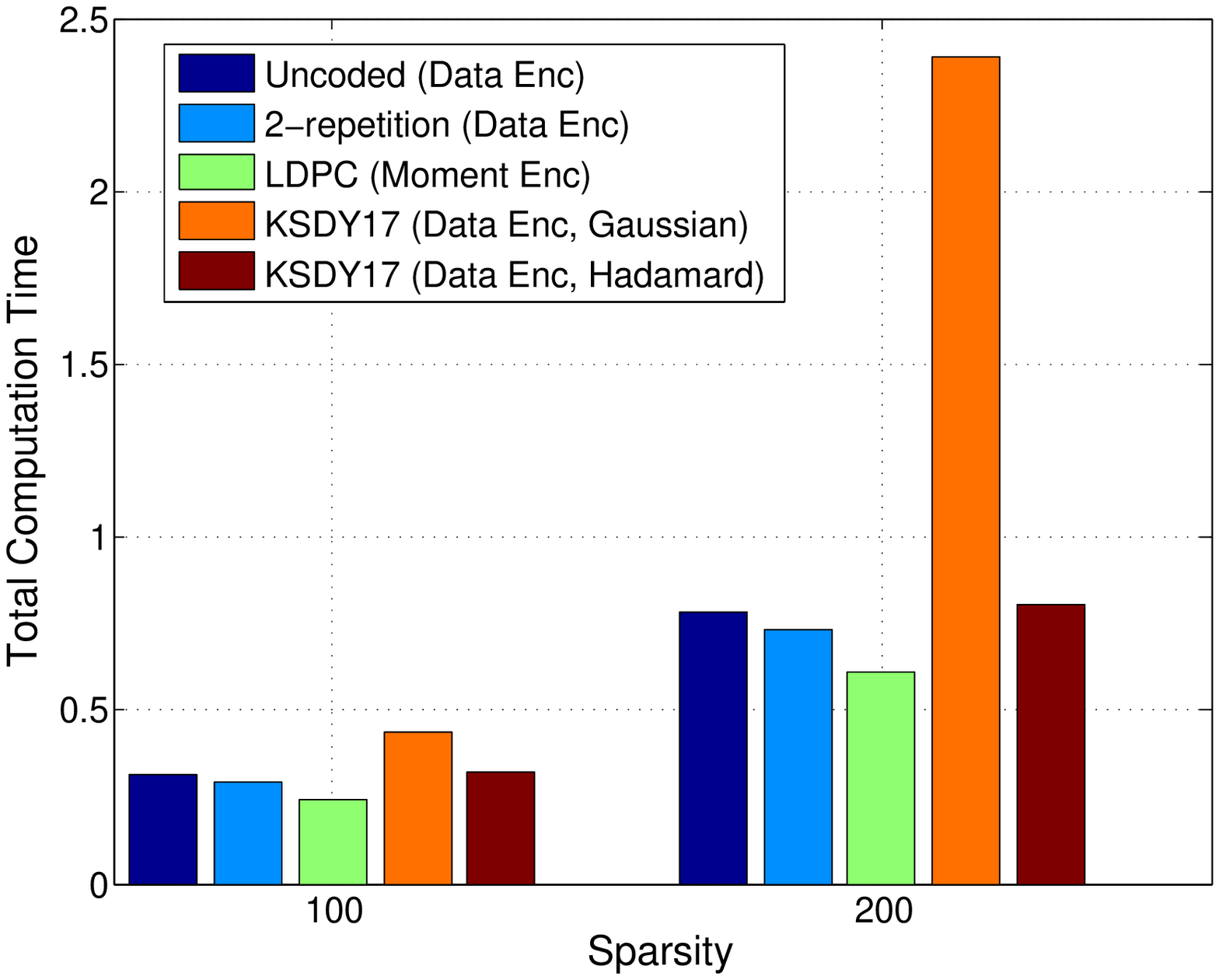}
                \label{fig:HDSparse35_time}
        \end{subfigure}%
       ~~
        \begin{subfigure}[b]{0.23\textwidth}
   		\footnotesize
                \centering
                \includegraphics[width=1\textwidth]{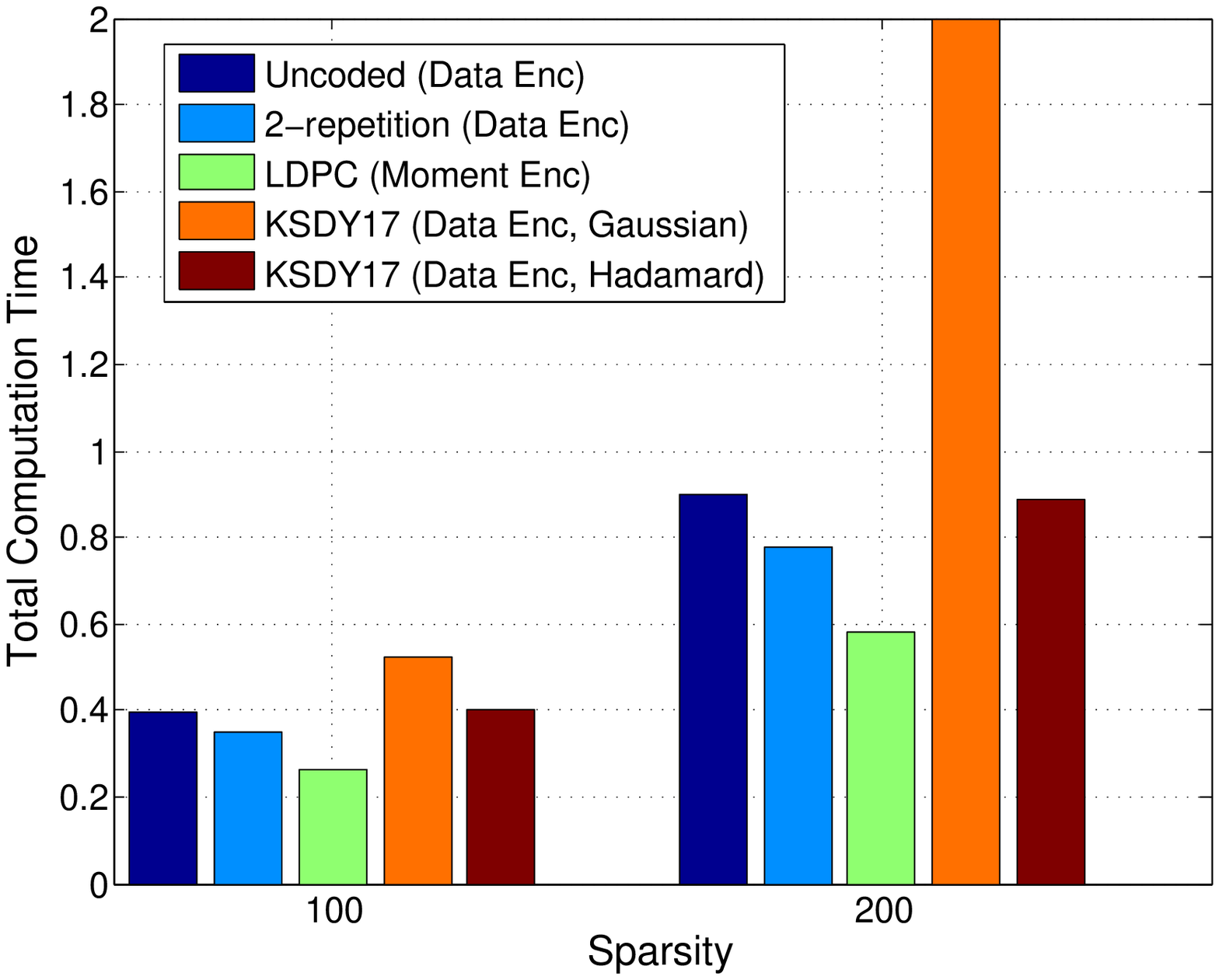}
                \label{fig:HDSparse30_time}
        \end{subfigure}
        \caption{\small Number of iterations and computation time for the sparse recovery problem in an underdetermined system $(k = 2000, m = 1024)$. The number of stragglers are 5, 10, 5 and 10 from left to right. \label{fig:SRunder}}                
\end{figure*}

Figure~\ref{fig:ls} presents the results for the least-square estimation problem. In our experiments, the data samples $X\equiv\{x_1,x_2,\ldots x_m \} \subset \mathbb{R}^k$ are randomly generated  with the dimensions $k \in \{200,400,800,1000\}$ and the number of total samples $m=2048$. The corresponding labels $\yv$ are created by multiplying the data matrix $X$ with randomly drawn vector $\bm{\theta}^{\ast} \in \mathbb{R}^k$. We implement 
Scheme \ref{sch:ldlc} on a real-life distributed computing framework (swarm2) at the University of Massachusetts Amherst  \cite{swarm2} using mpi4py Python package. The setup involves a cluster of $41$ computing nodes ($40$ worker nodes and $1$ master nodes). Throughout this section, the plotted results are averaged over $100$ trials. We compare our LDPC codes based (rate$=1/2$) moment encoding scheme with the recently proposed data encoding (with MDS/Gaussian matrices) scheme of Karakus et al. (KSDY17 in the figures)  \cite{KarakusSDY17}, as well as with uncoded and replication-based schemes ($2$-replication).\footnote{\color{black} Here, we do not compare our scheme with the approaches proposed in \cite{TLDK17} and \cite{LLPPR15} as both of these schemes involve significantly different computation and communication requirements (cf.~the supplementary material). For example, the gradient coding scheme~\cite{TLDK17} requires communicating $k$-dimensional vectors; and the approach of \cite{LLPPR15} involves encoding of two different matrices and two rounds of communications per step of the optimization procedure.} In all cases, we wait for either $30$ or $35$ workers to respond before the computations at the master node, i.e., the number of stragglers is $10$ or $5$, respectively. In order to implement our scheme, we utilize a $(40,20)$ LDPC code. In the replication-based schemes, we partition the data and repeat each partition of the data twice. We use sub-sampled Hadamard and Gaussian matrices to implement the data encoding method from \cite{KarakusSDY17}. We sampled the columns of $4096\times 4096$ Hadamard matrix and generated $4096\times 2048 $ random Gaussian matrices for the purpose of our experiments. For each case we record the number of steps until the Euclidean distance of the evaluated parameter from the actual parameter vector $\bm{\theta}^{\ast}$ is within a small threshold.

For the sparse recovery problem, we consider both the overdetermined $(m > k)$ and the underdetermined $(m < k)$ cases. For $m > k$, we adopt the same experimental setup as described above with $m = 2048$, but  restrict ourselves to the dimensions $k \in \{800, 1000\}$. For each  $k$, we consider different sparsities: for $f \in \{0.1, 0.2, 0.3, 0.4, 0.5\}$, $u = k\cdot f$ entries in $\bm{\theta}^{\ast}$ are nonzero. Figure~\ref{fig:SRover} presents the results for the sparse recovery problem in this overdetermined setup. We only plot the number of steps of the optimization procedure. The total computation time shows a similar trend.
For $m < k$, we generate the matrix $X$ 
as a $1024 \times 2000$ matrix with i.i.d. entries distributed according to the standard normal distribution. The true parameter vector $\bm{\theta}^{\ast}$ is drawn randomly with sparsity levels $u \in \{100, 200\}$. The results obtained from our experiments 
are presented in Figure~\ref{fig:SRunder}.        
As it is evident from the plots in Figure~\ref{fig:ls}, \ref{fig:SRover} and \ref{fig:SRunder}, our scheme requires smaller number of steps to converge to the true model parameters. Furthermore, our scheme also leads to smaller overall computation time.


\paragraph{Conclusion and future directions.}
In this paper we have proposed to encode the second moment of data for the purpose of running a distributed gradient descent algorithm. However our encoding is tailored for the 
squared-loss function (indeed, otherwise the second moment of data would not appear in the gradient). Our approach can be generalized to other  loss functions - such as logarithmic loss, or the Poisson loss function - relevant to various machine learning tasks. It would be an interesting future work to see what functional of the data needs to be encoded in those cases such that computation and communication overheads are minimized. 

%


%
\subsection*{Acknowledgements}

This work is supported in part by National Science Foundation  awards CCF 1642658 (CAREER) and  CCF 1618512.



\newpage

{
\bibliographystyle{plain}
\bibliography{Coded_computation}

\begin{thebibliography}{10}

\bibitem{AnGSS13}
G.~Ananthanarayanan, A.~Ghodsi, S.~Shenker, and I.~Stoica.
\newblock Effective straggler mitigation: Attack of the clones.
\newblock In {\em Proceedings of the 10th USENIX Conference on Networked
  Systems Design and Implementation (NSDI)}, pages 185--198, 2013.

\bibitem{BitarPR17}
R.~Bitar, P.~Parag, and S.~E. Rouayheb.
\newblock Minimizing latency for secure distributed computing.
\newblock In {\em Proceedings of 2017 IEEE International Symposium on
  Information Theory (ISIT)}, pages 2900--2904, 2017.

\bibitem{CharlesPE17}
Z.~Charles, D.~Papailiopoulos, and J.~Ellenberg.
\newblock Approximate gradient coding via sparse random graphs.
\newblock {\em arXiv preprint arXiv:1711.06771}, 2017.

\bibitem{DeanL13}
J.~Dean and L.~A. Barroso.
\newblock The tail at scale.
\newblock {\em Communications of the ACM}, 56(2):74--80, 2013.

\bibitem{MapReduce}
J.~Dean and S.~Ghemawat.
\newblock {MapReduce}: {S}implified data processing on large clusters.
\newblock {\em Communications of the ACM}, 51(1):107--113, Jan. 2008.

\bibitem{DuttaCG16}
S.~Dutta, V.~Cadambe, and P.~Grover.
\newblock Short-dot: Computing large linear transforms distributedly using
  coded short dot products.
\newblock In {\em Advances in Neural Information Processing Systems}, pages
  2100--2108, 2016.

\bibitem{DuttaCG17}
S.~Dutta, V.~Cadambe, and P.~Grover.
\newblock Coded convolution for parallel and distributed computing within a
  deadline.
\newblock {\em arXiv preprint arXiv:1705.03875}, 2017.

\bibitem{figueiredo2007gradient}
M.~Figueiredo, R.~D. Nowak, and S.~J. Wright.
\newblock Gradient projection for sparse reconstruction: Application to
  compressed sensing and other inverse problems.
\newblock {\em IEEE Journal of selected topics in signal processing},
  1(4):586--597, 2007.

\bibitem{gallager}
R.~Gallager.
\newblock Low-density parity-check codes.
\newblock {\em IRE Transactions on information theory}, 8(1):21--28, 1962.

\bibitem{garg2009gradient}
R.~Garg and R.~Khandekar.
\newblock Gradient descent with sparsification: an iterative algorithm for
  sparse recovery with restricted isometry property.
\newblock In {\em Proceedings of the 26th Annual International Conference on
  Machine Learning}, pages 337--344. ACM, 2009.

\bibitem{HalbawiRSH17}
W.~Halbawi, N.~Azizan Ruhi, F.~Salehi, and B.~Hassibi.
\newblock Improving distributed gradient descent using {Reed-Solomon} codes.
\newblock {\em arXiv preprint arXiv:1706.05436}, 2017.

\bibitem{kakade2011efficient}
S.~M. Kakade, V.~Kanade, O.~Shamir, and A.~Kalai.
\newblock Efficient learning of generalized linear and single index models with
  isotonic regression.
\newblock In {\em Advances in Neural Information Processing Systems}, pages
  927--935, 2011.

\bibitem{KarakusSDY17}
C.~Karakus, Y.~Sun, S.~Diggavi, and W.~Yin.
\newblock Straggler mitigation in distributed optimization through data
  encoding.
\newblock In {\em Advances in Neural Information Processing Systems}, pages
  5440--5448, 2017.

\bibitem{koltchinskii2011nuclear}
V.~Koltchinskii, K.~Lounici, and A.~B. Tsybakov.
\newblock Nuclear-norm penalization and optimal rates for noisy low-rank matrix
  completion.
\newblock {\em The Annals of Statistics}, 39(5):2302--2329, 2011.

\bibitem{LLPPR15}
K.~Lee, M.~Lam, R.~Pedarsani, D.~Papailiopoulos, and K.~Ramchandran.
\newblock Speeding up distributed machine learning using codes.
\newblock {\em IEEE Transactions on Information Theory}, 64(3):1514--1529,
  March 2018.

\bibitem{LeeSR17}
K.~Lee, C.~Suh, and K.~Ramchandran.
\newblock High-dimensional coded matrix multiplication.
\newblock In {\em Proceedings of IEEE International Symposium on Information
  Theory (ISIT)}, pages 2418--2422, 2017.

\bibitem{LiMA15}
S.~Li, M.~A. Maddah-Ali, and A.~S. Avestimehr.
\newblock Coded {M}apreduce.
\newblock In {\em Proceedings of 53rd Annual Allerton Conference on
  Communication, Control, and Computing (Allerton)}, pages 964--971, 2015.

\bibitem{LiMA16}
S.~Li, M.~A. Maddah-Ali, and A.~S. Avestimehr.
\newblock A unified coding framework for distributed computing with straggling
  servers.
\newblock In {\em Proceedings of IEEE Globecom Workshops (GC Wkshps)}, pages
  1--6, 2016.

\bibitem{LiMYA17}
S.~Li, M.~A. Maddah-Ali, Q.~Yu, and A.~S. Avestimehr.
\newblock A fundamental tradeoff between computation and communication in
  distributed computing.
\newblock {\em IEEE Transactions on Information Theory}, 64(1):109--128, Jan
  2018.

\bibitem{mairal2009online}
J.~Mairal, F.~Bach, J.~Ponce, and G.~Sapiro.
\newblock Online dictionary learning for sparse coding.
\newblock In {\em Proceedings of the 26th annual international conference on
  machine learning}, pages 689--696. ACM, 2009.

\bibitem{Nemi09}
A.~Nemirovski, A.~Juditsky, G.~Lan, and A.~Shapiro.
\newblock Robust stochastic approximation approach to stochastic programming.
\newblock {\em SIAM Journal on Optimization}, 19(4):1574--1609, 2009.

\bibitem{plan2016generalized}
Y.~Plan and R.~Vershynin.
\newblock The generalized lasso with non-linear observations.
\newblock {\em IEEE Transactions on information theory}, 62(3):1528--1537,
  2016.

\bibitem{RavivTTD17}
N.~Raviv, I.~Tamo, R.~Tandon, and A.~G. Dimakis.
\newblock Gradient coding from cyclic {MDS} codes and expander graphs.
\newblock {\em arXiv preprint arXiv:1707.03858}, 2017.

\bibitem{MCT08}
T.~Richardson and R.~L. Urbanke.
\newblock {\em Modern Coding Theory}.
\newblock Cambridge University Press, New York, NY, USA, 2008.

\bibitem{RU01}
T.~J. Richardson and R.~L. Urbanke.
\newblock The capacity of low-density parity-check codes under message-passing
  decoding.
\newblock {\em IEEE Transactions on Information Theory}, 47(2):599--618, Feb
  2001.

\bibitem{ShahLR16}
N.~B. Shah, K.~Lee, and K.~Ramchandran.
\newblock When do redundant requests reduce latency?
\newblock {\em IEEE Transactions on Communications}, 64(2):715--722, Feb 2016.

\bibitem{UnderstandingML}
S.~Shalev-Shwartz and S.~Ben-David.
\newblock {\em Understanding Machine Learning: From Theory to Algorithms}.
\newblock Cambridge University Press, New York, NY, USA, 2014.

\bibitem{SS96}
M.~Sipser and D.~A. Spielman.
\newblock Expander codes.
\newblock {\em IEEE Transactions on Information Theory}, 42(6):1710--1722, Nov
  1996.

\bibitem{swarm2}
Swarm2.
\newblock Swarm user documentation.
\newblock
  \url{https://people.cs.umass.edu/~swarm/index.php?n=Main.NewSwarmDoc}, 2018.
\newblock Accessed: 2018-01-05.

\bibitem{TLDK17}
R.~Tandon, Q.~Lei, A.~G. Dimakis, and N.~Karampatziakis.
\newblock Gradient coding: Avoiding stragglers in distributed learning.
\newblock In {\em Proceedings of the 34th International Conference on
  International Conference on Machine Learning (ICML)}, pages 3368--3376, 2017.

\bibitem{tibshirani2015statistical}
R.~Tibshirani, M.~Wainwright, and T.~Hastie.
\newblock {\em Statistical learning with sparsity: the lasso and
  generalizations}.
\newblock Chapman and Hall/CRC, 2015.

\bibitem{WJW15}
D.~Wang, G.~Joshi, and G.~Wornell.
\newblock Using straggler replication to reduce latency in large-scale parallel
  computing.
\newblock {\em ACM SIGMETRICS Performance Evaluation Review}, 43(3):7--11,
  2015.

\bibitem{YangGK17}
Y.~Yang, P.~Grover, and S.~Kar.
\newblock Coding method for parallel iterative linear solver.
\newblock {\em arXiv preprint arXiv:1706.00163}, 2017.

\bibitem{YuMA17}
Q.~Yu, M.~A. Maddah-Ali, and A.~S. Avestimehr.
\newblock Polynomial codes: an optimal design for high-dimensional coded matrix
  multiplication.
\newblock {\em arXiv preprint arXiv:1705.10464}, 2017.

\bibitem{Spark}
M.~Zaharia, M.~Chowdhury, M.~J. Franklin, S.~Shenker, and I.~Stoica.
\newblock {Spark}: {C}luster computing with working sets.
\newblock In {\em Proceedings of the 2nd USENIX Conference on Hot Topics in
  Cloud Computing (HotCloud)}, pages 10--10, 2010.

\end{thebibliography}
}

%
%
%
%
%
%
%
%




\end{document}